\documentclass[11pt]{article}

\usepackage[margin=1in]{geometry}
\usepackage{enumitem}
\usepackage[utf8]{inputenc}   
\usepackage[T1]{fontenc}      
\usepackage{url}              
\usepackage{booktabs}         
\usepackage{amsfonts}         
\usepackage{nicefrac}         
\usepackage[expansion=false]{microtype}  
\usepackage{xcolor}           

\usepackage{amsmath,amssymb,amsthm,mathtools,multirow}
\usepackage{graphicx}
\usepackage{subcaption}

\usepackage{hyperref}
\hypersetup{colorlinks=true, allcolors=blue}

\DeclarePairedDelimiterX{\infdivx}[2]{(}{)}{%
  #1\;\delimsize\|\;#2}
\DeclareMathOperator{\var}{\mathrm{Var}}

\newcommand{\abs}[1]{\left\lvert#1\right\rvert}

\renewcommand{\epsilon}{\varepsilon}

\newcommand{\X}{\mathcal{X}}

\newcommand{\cost}{\texttt{cost}}

\newcommand{\dist}{\partial}

\renewcommand{\P}{\mathcal{P}}
\renewcommand{\S}{\mathcal{S}}
\newcommand{\K}{\textbf{K}}
\newcommand{\parameterspc}{\mathcal{C}}
\renewcommand{\theta}{C}
\renewcommand{\eta}{y}
\newcommand{\totaldim}{\ell}

\theoremstyle{plain}

\newtheorem{theorem}{Theorem}[section]

\newtheorem{corollary}[theorem]{Corollary}
\newtheorem{proposition}[theorem]{Proposition}

\theoremstyle{definition}
\newtheorem{definition}[theorem]{Definition}

\newtheorem{remark}{Remark}

\newcommand{\DPP}{~determinantal sampler~}

\title{Beyond Worst-Case Coreset Bounds for $k$-Clustering via Determinantal Sampling}

\author{%
  Diptarka Chakraborty$^1$ \qquad Satyaki Mukherjee$^2$ \\
  Gaurav Vallabhdas Revankar$^3$ \qquad Hoang-Son Tran$^2$\\[6pt]
  $^1$School of Computing, National University of Singapore\\
  $^2$Department of Mathematics, National University of Singapore\\
  $^3$Department of Mathematics, Indian Institute of Technology Bombay
}
\date{}

\begin{document}

\maketitle

\begin{abstract}
Massive datasets in modern machine learning have made data reduction a central challenge, particularly for clustering tasks where memory and computational constraints demand compact yet faithful summaries. A standard approach is to construct an \emph{$\epsilon$-coreset}: a small weighted subset that approximately preserves the clustering cost for every plausible choice of centers. For the \emph{$(k,z)$-clustering problem}, existing worst-case bounds on coreset size are essentially tight, ruling out substantially smaller coresets in general. However, such worst-case instances are often unrepresentative of real-world data. 

In this work, we show that significantly smaller coresets are possible under mild and natural assumptions on the underlying data distribution. We introduce a new correlated sampling framework, called \emph{determinantal sampling}, based on a novel application of determinantal point processes. 
Using this framework, we obtain an efficiently constructible
\(\varepsilon\)-coreset for \((k,z)\)-clustering in \(\mathbb R^d\) whose
dependence on \(1/\varepsilon\) has exponent strictly smaller than \(2\) when
\(d\) is fixed. This improves over the worst-case \(\varepsilon^{-2}\) barrier
under our beyond-worst-case assumptions. 
To the best of our knowledge, this is the first result that provably surpasses these lower bounds through beyond-worst-case assumptions. Finally, we validate our approach on synthetic and real-world benchmark datasets, where it consistently achieves smaller coresets than existing state-of-the-art methods, even without explicitly enforcing the assumptions used in the analysis.
\end{abstract}

\section{Introduction}
\label{sec:intro}
The proliferation of massive datasets in modern machine learning applications necessitates the development of robust data reduction and summarization techniques. These methods are essential for minimizing memory usage and enabling distributed computations across systems with limited local storage. A prevalent strategy for managing such massive data involves sparsifying the input while meticulously preserving its underlying structural properties. This ensures that subsequent machine learning tasks, whether supervised or unsupervised, can be executed (nearly) accurately using only the summarized subset. Among the most prominent data summarization paradigms is the \emph{coreset}. Since its introduction, coresets have garnered substantial attention in recent years for their ability to facilitate highly efficient algorithms in terms of memory, time, and parallel processing, with applications in streaming~\cite{har2004coresets, frahling2005coresets, braverman2016new, braverman2019streaming}, distributed~\cite{balcan2013distributed, huang2022coresets}, and dynamic settings~\cite{henzinger2020fully, wang2021robust}. Their utility spans diverse domains, including computational geometry~\cite{agarwal2004approximating}, graph sparsification~\cite{spielman2008graph, batson2009twice}, linear and logistic regression~\cite{dasgupta2009sampling, munteanu2018coresets}, clustering~\cite{har2004coresets, feldman2011unified, cohen2022improved, huang2023coresets}, classification~\cite{mai2021coresets}, subspace approximation~\cite{woodruff2023sharper}, and learning mixture models~\cite{lucic2018training}.

Clustering remains a fundamental data processing task with pervasive applications across fields such as image segmentation, document categorization, anomaly detection, and healthcare analytics, e.g., see~\cite{lloyd1982least, tan2006cluster, ArthurV07, coates2012learning}. One of the well-studied variants is the \emph{$(k,z)$-clustering} that, given a set $P$ of points from a metric space $\X$ and a non-negative integer $k$, asks to find a set $C$ of at most $k$ centers from the underlying metric space $\X$ that minimizes the cost, defined as $\cost_z(P, C):= \sum_{x \in P} \dist^z(x, C)$, where $\dist(x, C)$ denotes the minimum distance from $x$ to a center point in $C$. Notable special cases of this variant include the classical \emph{$k$-median} clustering when $z=1$, and the ubiquitous \emph{$k$-means} clustering in Euclidean space when $z=2$. 

Given the massive scale of modern datasets, it is highly advantageous to extract a (weighted) representative subset that faithfully approximates the clustering cost for any arbitrary set of centers. Specifically, this subset must preserve the cost within a multiplicative approximation factor of $(1+\epsilon)$ for any $\epsilon > 0$. Such a representative subset is formally known as an \emph{$\epsilon$-coreset} (see Section~\ref{sec:prelims} for a formal definition). Extensive research in this area has established a tight theoretical understanding of coreset sizes for the $(k,z)$-clustering problem~\cite{chen2009coresets, langberg2010universal, becchetti2019oblivious, feldman2020turning, huang2020coresets}. The best-known upper bound on the size of such a coreset is $\tilde{O}_z(k \epsilon^{-z-2})$~\cite{cohen2021new, cohen2022improved}, where $\tilde{O}_z(\cdot)$ hides logarithmic factors and dependencies only on $z$. Later, a similar bound is also achieved using \emph{sensitivity sampling} for Euclidean $k$-means (i.e., for $z = 2$)~\cite{bansal2024sensitivity}. The above bound is known to be essentially tight for Euclidean spaces. More specifically, a lower bound of $\Omega(k \epsilon^{-2})$ for the $k$-means~\cite{cohen2022towards} and for arbitrary $z$, a lower bound of $\Omega(k \epsilon^{-z-2})$ for $\epsilon \ge \Omega\left(k^{-1/(z+2)}\right)$~\cite{huang2024optimal}, are already established. 

These bounds definitely preclude the existence of smaller coresets in the worst-case scenario. However, empirical data encountered in real-world applications rarely exhibit worst-case behavior. That motivates the researchers to explore various beyond-worst-case assumptions to exploit when designing efficient clustering algorithms~\cite{kumar2010clustering, ostrovsky2013effectiveness, balcan2013clustering, angelidakis2017algorithms, cohen2017local}. However, to the best of our knowledge, only recently,~\cite{bansal2024sensitivity} looked into coreset construction from the beyond-worst-case perspective -- more specifically, with \emph{cost-stable clusterability criteria} from~\cite{ostrovsky2013effectiveness}. Unfortunately, even with that assumption, it is not possible to surpass the $\Omega(k \epsilon^{-2})$ lower bound~\cite{bansal2024sensitivity}. 

It naturally motivates a compelling research question: Is it possible to circumvent these worst-case lower bounds by introducing realistic or natural assumptions about the distribution or structure of the input points? In this work, we provide an affirmative answer to this question, demonstrating that beyond-worst-case analysis can yield significantly more compact coresets. 

\noindent \textbf{Our Contribution. }
The primary contribution of this paper is an efficient coreset construction for the $(k,z)$-clustering problem that, under mild assumptions on the underlying data distribution, improves upon the known worst-case lower bound. Our approach is based on a correlated sampling method, which we call \emph{determinantal sampling}, arising from a novel application of determinantal point processes. More precisely, we establish the following informal theorem.

\begin{theorem}[Informal]
\label{thm:main-informal}
Under certain mild assumptions on the underlying data distributions, as specified in Section~\ref{sec:coreset-bound}, the $(k,z)$-clustering problem on $\mathbb{R}^d$ admits an efficiently constructible $\epsilon$-coreset of size $\tilde{O}_{d,z}(k^{d/(d+1)} \epsilon^{-2d/(d+1)})$.
\end{theorem}

Here, $\tilde{O}_{d,z}(\cdot)$ hides polylogarithmic factors as well as factors depending only on $d$ and $z$. In particular, for constant dimension $d$, this bound improves upon the worst-case lower bound of $\Omega(\epsilon^{-2})$ established in~\cite{cohen2022towards}. This is especially relevant in practice, since many real-world datasets are effectively low-dimensional. Even when represented in a high-dimensional ambient space, by applying standard dimensionality-reduction techniques such as the Johnson--Lindenstrauss lemma, PCA, and related tools, one can often work in a reduced-dimensional space without significantly distorting the clustering structure. Consequently, our coreset construction is well-suited for practical scenarios in which the effective dimension is small, allowing one to beat the known worst-case bound. We also note that to derive Theorem~\ref{thm:main-informal}, we state and prove a more general result in Theorem~\ref{thm:kz-clustering-dpp} for any Polish space. As such, similar results hold in more generality as long as one can find DPPs that satisfy the required variance bounds.

Our work is motivated by the observation that real-world datasets arising in machine learning are unlikely to exhibit worst-case behavior. This has already led to recent attempts to obtain smaller coresets under structural or distributional assumptions, initiated in~\cite{bansal2024sensitivity}. However, the cost-stable clusterability conditions introduced there are not sufficient to overcome the worst-case lower bound. Ours is the first to provably surpass that lower bound by leveraging natural assumptions, like subgaussianity, on the data distribution. Although our analysis relies on certain assumptions, they are mild, well-motivated, and commonly satisfied in practice (see Section~\ref{sec:coreset-bound}).

On the technical side, we introduce a new correlated sampling framework, called \emph{determinantal sampling}, based on a novel use of determinantal point processes. The key idea is to exploit correlations in the sampling process to obtain stronger coreset guarantees than those achievable with standard independent sampling methods.
An additional practical benefit is that once such a kernel has been constructed, it can be reused to efficiently generate multiple coresets. This is particularly appealing in applications where repeated sampling is desirable, such as improving robustness, obtaining stronger performance guarantees for downstream clustering tasks, attaining stability over noise in ensemble-style clustering, and applications requiring reliability, uncertainty estimation, or hyperparameter selection (e.g.,
~\cite{yu2019bootstrapping, dudoit2003bagging, zhang2025generalization}).

To make the framework concrete and evaluate its practical performance, we implement several DPP kernel constructions tailored to the clustering setting, and test them on both synthetic and real-world benchmark datasets. 
The synthetic experiments support the improved \(\varepsilon\)-dependence predicted by our analysis, while the real-world benchmark experiments test the method without explicitly enforcing the assumptions required by the theory. Even in this setting, our coreset construction outperforms state-of-the-art methods in terms of coreset size for comparable clustering error. The results also highlight the modular role of the kernel: our theorem applies whenever the induced DPP sampler satisfies the required variance conditions, while the experiments include additional kernels for which we do not yet prove these guarantees. Their strong empirical performance suggests that determinantal sampling may be effective beyond the regime currently covered by our analysis.

\section{Preliminaries}
\label{sec:prelims}

\noindent \textbf{Notations. }
Throughout this paper, we consider \emph{Polish spaces} which are separable, complete metric spaces, as defined in \cite{Cohn1980}. From now on, the metric spaces we consider in this paper are assumed to be separable and complete. It is worth remarking that most general metric spaces satisfy the separability and completeness property; for more details, see a discussion in Appendix~\ref{app:polish}.

For a metric space $(\mathcal{X}, \dist)$, a subset $N \subseteq \mathcal{X}$ is a
\emph{$u$-net} if every point in $\mathcal{X}$ is within distance $u$ of some
point in $N$. The \emph{covering number} $\mathcal{N}(\mathcal{X}, u, \dist)$ denotes
the minimum size of a $u$-net for $(\mathcal{X}, \dist)$.

\noindent \textbf{$(k,z)$-Clustering and Coresets.}
The \emph{\((k,z)\)-clustering} problem is defined as follows: Let \((\mathcal X,\dist)\)
be a metric space, where \(\mathcal X\) could be either continuous or discrete.
Given a dataset \(\mathcal{P}\subseteq\mathcal X\) of \(n\) points and a parameter
\(z\ge 1\), the goal is to choose a set \(\theta\subseteq\mathcal X\) of (at most)
\(k\) centers minimizing
\begin{equation}
\operatorname{cost}_z(\mathcal{P},\theta)
:=
\sum_{x\in \mathcal{P}} \dist(x,\theta)^z,
\label{eq:kz-clustering}
\end{equation}
where $\dist(x,\theta):=\min_{\eta\in\theta} \dist(x,\eta)$ denotes the distance from \(x\) to its nearest center in \(\theta\).

A coreset is a data-reduction technique that approximates an optimization objective using a small
(weighted) subset of the data.
In the context of $(k,z)$-Clustering, let $\mathcal{X}^k$ denote the set of all the $k$-sized subsets (potential center set) of $\mathcal{X}$.
A coreset is defined as follows.

\begin{definition}[Coreset {\cite{langberg2010universal,feldman2011unified}}]
Consider a metric space $(\mathcal{X},\dist)$. For a dataset $\mathcal{P}\subseteq\mathcal{X}$ of $n$ points,
an \emph{$\varepsilon$-coreset} for the $(k,z)$-Clustering problem is a weighted subset $\S\subseteq \mathcal{P}$ with
weight $w:\S\to\mathbb{R}_{\ge 0}$ such that
\[
\forall \theta\in \mathcal{X}^k,\qquad
\sum_{x\in \S} w(x)\cdot \dist(x,\theta)^z
\in (1\pm\varepsilon)\cdot \operatorname{cost}_z(\mathcal{P},\theta).
\]
\end{definition}

The goal is to give good upper bounds for the size of the smallest possible $\varepsilon$-coreset. 

\noindent \textbf{Samplers via determinantal point processes. }
Let $\mathcal{P}$ be a Polish space equipped a reference measure $\nu$. We recall that a point process $\mathcal{S}$ on $\mathcal{P}$ is a random locally finite subset of $\mathcal{P}$. The standard definition of determinantal point processes (\cite{Mac75, HKPV06}) is as follows:

\begin{definition}
A point process $\mathcal{S}$ on $\mathcal{P}$ is called a \emph{determinantal point process} (DPP) with respect to $\nu$ if there exists a measurable kernel $K:\mathcal{P}\times\mathcal{P}\to\mathbb{C}$ such that, for all
$k\in\mathbb{N}$ and all bounded measurable functions $f:\mathcal{P}^k\to\mathbb{R}$, we have
\begin{equation}\label{eq:DPPdefn}
\mathbb{E}\!\Big [\sum_{\underset{x_{1} \ne \ldots \ne x_{k}}{x_{1},\ldots,x_{k}\in \mathcal{S}}}
f(x_{1},\ldots,x_{k})\Big]
=
\int_{\mathcal{P}^k}
f(x_1,\ldots,x_k)\,
\det\!\bigl[K(x_i,x_j)\bigr]_{i,j=1}^k\,
\mathrm{d}\nu^{\otimes k}(x_1,\ldots,x_k).
\end{equation}
The function $\K$ is called a \emph{kernel} of the DPP, and $\nu$ the \emph{background measure}. If $\K(x,y) = \overline{\K(y,x)}$, we call $\K$ a Hermitian kernel.
\end{definition}

An important setting in machine learning is when the background space $\mathcal{P}$ is a finite set, which could be intepreted as a set of data points. The standard reference measure on $\mathcal{P}$ is the counting measure. In this discrete setting, DPPs have the following simple characterization: 

\begin{definition}[Discrete DPPs]
    Let $\mathcal{P}$ be a finite set of $n$ points. A random subset $\mathcal{S} \subset \mathcal{P}$ is a DPP on $\mathcal{P}$ if there exists an $n\times n$-matrix $\K$ (indexed by elements of $\mathcal{P}$) such that 
    \[
\mathbb{P}(T\subseteq \mathcal{S})=\det(\K_T), \qquad \forall\, T\subseteq\mathcal{P},
\]
where $\K_T$ denotes the submatrix of $\K$ indexed by the subset $T\subset \mathcal{P}$.
\end{definition}

Determinantal point processes (DPPs) have emerged as an attractive framework for repulsive sampling in machine learning, and several advanced algorithms have been developed for efficient DPP sampling. For further discussion on DPPs, particularly from a machine learning perspective, we refer to the survey \cite{kulesza2012determinantal} and the recent survey \cite{tran2025negative}.

\noindent \textbf{DPPs as coresets. }Given a dataset $\P = \{X_1,\dots,X_n\}\subseteq\mathcal{X}$  consisting of $n$ iid samples from a distribution $\mu$ and a function $f:\mathcal{X}\to\mathbb{R}_{\ge 0}$, define
\begin{equation}
    L(f) :=  \sum_{i=1}^{n} f(X_i).
    \label{eq:empirical-mean}
\end{equation}
Given a coreset $\mathcal{S}\subseteq P$ with associated weights $w:\mathcal{S}\to\mathbb{R}_{\ge 0}$, define the \emph{weighted coreset estimator} as $L_{\mathcal{S}}(f) := \sum_{x\in\mathcal{S}} w(x)\,f(x)$.
We can now define a weighted coreset estimator using a Determinantal Point Process as follows.

\begin{definition}[Determinantal Coreset Estimator]
Sample $\mathcal{S} \subset \mathcal{P}$ by a DPP with Kernel $\mathbf{K}$ and with reference measure $\nu$ as the counting measure on $\mathcal{P}$. We choose the weight $\omega(x) = \mathbf{K}(x, x)^{-1}$. Thus $L_{\mathcal{S}}(f) = \sum_{x\in\mathcal{S}}  \mathbf{K}(x, x)^{-1}f(x)$. By \eqref{eq:DPPdefn}, this choice of $\omega$ makes $L_{\mathcal{S}}(f)$ an unbiased estimator for $L(f)$. In this case, we say that $\S$ is sampled by a \emph{\DPP}.
\end{definition} 

\section{Coreset Construction beyond Worst-Case}
\label{sec:coreset-bound}

\subsection{General determinantal sampler guarantees}

We begin with a uniform coreset bound for a general parametric loss class under determinantal sampling, which we later specialize to clustering objectives. Let
$
\mathcal F=\{f_\theta:\theta\in\parameterspc\}
$
be a class of nonnegative functions indexed by a metric space \((\parameterspc,\dist_\parameterspc)\). (Note that the space \((\parameterspc,\dist_\parameterspc)\) is different from the metric space $(\X,d)$ where the points in datasets belong.) We assume:
\begin{enumerate}[label=\textbf{(A.\arabic*)}]
  \item \label{asm:class}
    The parameter space \((\parameterspc,\dist_\parameterspc)\) is bounded and has parametric covering dimension \(\totaldim\): there exists \(c_\parameterspc>0\) such that
    \[
    \mathcal N(\parameterspc,u,\dist_\parameterspc)\le \left(\frac{c_\parameterspc}{u}\right)^{\totaldim},
    \qquad 0<u\le 1.
    \]

  \item \label{asm:lipschitz}
    The map \(\theta\mapsto f_\theta(x)^{1/z}\) is \(1\)-Lipschitz uniformly in \(x\): for all
    \(\theta,\theta'\in\parameterspc\) and all \(x\in\mathcal X\),
    \[
    \left|f_\theta(x)^{1/z}-f_{\theta'}(x)^{1/z}\right|
    \le \dist_\parameterspc(\theta,\theta').
    \]
\end{enumerate}

For the \((k,z)\)-clustering loss, these assumptions are natural. Taking
$
f_\theta(x)=\dist(x,\theta)^z,
$
where \(\theta\) denotes a \(k\)-tuple of centers, we verify in the next section that \(\mathcal F\) satisfies Assumptions~\ref{asm:class} and~\ref{asm:lipschitz} with \(\totaldim=dk\). Thus \(\totaldim\) can be viewed as the effective dimension of the parameter space; in the standard Euclidean setting, it coincides with the usual parameter dimension.

For technical reasons, we will also need some assumption which is essentially equivalent to the requirement that the dataset $\{x_1,..., x_n\}$ is non-degenerate. We will later show (in Theorem~\ref{thm:app-empirical-bound}) that this ensuing assumption will hold with high probability when the data is, for instance, drawn independently from a subgaussian measure.
\begin{enumerate}[label=\textbf{(B.\arabic*)}]
   \item \label{asm:mult error}  
    $\frac{1}{n}|L(f)|\geq c$, for some $c>0$, uniformly on $\mathcal{F}$.
\end{enumerate}

\begin{theorem}[DPP Coreset Bound for General $(k,z)$-Clustering]
\label{thm:kz-clustering-dpp} 
Let \(\mathcal S\) be drawn from a  \DPP with a Hermitian kernel \(\mathbf K\) and with respect to the counting measure of a finite set \(\mathcal P=\{x_1,\dots,x_n\}\subset \mathcal{X}\). Let \(m=\mathbb E[|\mathcal S|]\). Assume that for all \(i\in\{1,\dots,n\}\), \(\mathbf K(x_i,x_i)\ge \rho\cdot m/n\) for some \(\rho>0\) not depending on \(m,n\). Suppose $V \geq \sup_{\theta\in\parameterspc}\mathrm{Var}\!\left[\frac1nL_{\mathcal S}(f_\theta)\right]$ and \(|\mathcal S|\le B\cdot m\) almost surely for some \(B>0\). Then under \ref{asm:class}, \ref{asm:lipschitz},\ref{asm:mult error} and  for some constants $c_1, c_2, c_3$ depending on $\parameterspc,B,\rho,z,c$ we have, for any $0\le \varepsilon\le
\frac{c_3 mV}{\sup_{\theta\in\parameterspc}\|f_\theta\|_\infty}$,
\[
\mathbb P\!\left(\exists f\in\mathcal F:\left|\frac{L_{\mathcal S}(f)}{L(f)}-1\right|\ge \varepsilon\right)
\le
2\exp\!\left(c_1\totaldim-\totaldim\log\varepsilon-\frac{c_2\varepsilon^2}{V}\right).
\]

\end{theorem}

\begin{proof}
Define 
\[
\mathcal{E}(\theta):=\left(\frac{L_{\mathcal{S}}(f_\theta)}{n}\right)^{1/z}-\left(\frac{L(f_\theta)}{n}\right)^{1/z}, \qquad \tilde\varepsilon:=\varepsilon c^{1/z}/(2z).
\]
Then a simple algebraic inequality (detailed in Proposition~\ref{prop:alg_ineq_1} in appendix) implies that under assumption~\ref{asm:mult error} we have,
\[
\mathbb{P}\!\left(\exists\,f\in\mathcal{F}:\left|\frac{L_{\mathcal{S}}(f)}{L(f)}-1\right|\ge\varepsilon\right)
\le
\mathbb{P}\!\left(\sup_{\theta\in\parameterspc}|\mathcal{E}(\theta)|\ge\tilde\varepsilon\right).
\]

Let $\theta,\theta'\in\parameterspc$.
By Assumption~\ref{asm:lipschitz},
$|f_\theta(x)^{1/z}-f_{\theta'}(x)^{1/z}|\le \dist_\parameterspc(\theta,\theta')$
for all $x\in\mathcal{P}$.
Thus by Minkowski's inequality applied to $L_{\mathcal{S}}(f)/n = \sum_{x\in\mathcal{S}}  \mathbf{K}(x, x)^{-1}n^{-1}f(x)$,
\[
\Bigl|(L_\mathcal{S}(f_\theta)/n)^{1/z} - (L_\mathcal{S}(f_{\theta'})/n)^{1/z}\Bigr|
\le \dist_\Theta(\theta,\theta')
    \Bigl(\sum_{x\in\mathcal{S}}\tfrac{1}{nK(x,x)}\Bigr)^{1/z}
\]
Since $\mathbf{K}(x,x)\ge\rho m/n$ by~(i) and $|\mathcal{S}|\le Bm$, we have
 a.s.\,
\[
\abs{\left(L_{\mathcal{S}}(f_{\theta})/n\right)^{1/z}-\left(L_{\mathcal{S}}(f_{\theta'})/n\right)^{1/z}}
\le
(B/\rho)^{1/z}\,\dist_\parameterspc(\theta,\theta').
\]
The same argument with unit weights gives
$\abs{\left(L(f_{\theta})/n\right)^{1/z}-\left(L(f_{\theta'})/n\right)^{1/z}}
\le \dist_\parameterspc(\theta,\theta')$,
and therefore
\[
|\mathcal{E}(\theta)-\mathcal{E}(\theta')|
\le c_0\,\dist_\parameterspc(\theta,\theta')\quad\text{almost surely with},
\qquad
c_0 := 1+(B/\rho)^{1/z}.
\]

Let $\Gamma$ be a $\tilde\varepsilon/(2c_0)$-net for $\parameterspc$. Then
$
\sup_{\theta\in\parameterspc}|\mathcal{E}(\theta)|
\le
\sup_{\theta'\in\Gamma}|\mathcal{E}(\theta')|+\frac{\tilde\varepsilon}{2},
$
so
$
\mathbb{P}\!\left(\sup_{\theta\in\parameterspc}|\mathcal{E}(\theta)|\ge\tilde\varepsilon\right)
\le
\sum_{\theta'\in\Gamma}
\mathbb{P}\!\left(|\mathcal{E}(\theta')|\ge\frac{\tilde\varepsilon}{2}\right).
$

Fix $\theta'\in\Gamma$. Again some simple algebra detailed in Proposition~\ref{prop:alg_ineq_1} (in appendix) implies that under assumption~\ref{asm:mult error} we have,
\begin{equation}   \label{eq:Etheta_less_LsfLf}
\mathbb{P}\!\left(|\mathcal{E}(\theta')|\ge\frac{\tilde\varepsilon}{2}\right)
\le
\mathbb{P}\!\left(\frac{1}{n}\bigl|L_{\mathcal{S}}(f_{\theta'})-L(f_{\theta'})\bigr|
\ge
\frac{\tilde\varepsilon}{2}\,c^{(z-1)/z}\right).
\end{equation}

 Using equation~\eqref{eq:DPPdefn} it is easy to verify that $\mathbb{E}\left[L_{\mathcal{S}}(f_{\theta'})\right] = L(f_{\theta'})$. We can thus conclude the results by using standard concentration results on Determinantal Samplers given by e.g. \cite{bardenet2024small},\cite{breuer2014nevai}.  
In particular using Theorem~\ref{thm:thm1_bardenet}~(Theorem~1 of \cite{bardenet2024small}) stated in appendix~\ref{app:conc_bound} with $\phi(x) = f(x) / K(x,x)$ to the R.H.S. of equation~\eqref{eq:Etheta_less_LsfLf}, we obtain
\[
\mathbb{P}\!\left(|\mathcal{E}(\theta')|\ge\frac{\tilde\varepsilon}{2}\right)
\le
2\exp\!\left(-\frac{c^{2(z-1)/z}\tilde\varepsilon^2}{16A V}\right)
\]
Using $\mathbf{K}(x,x)\ge\rho m/n$ the above inequality holds for all $0\le\tilde\varepsilon\le 4A\rho m V/(3c^{(z-1)/z}\|f_\theta\|_\infty)$.

 Substituting $\tilde\varepsilon=\varepsilon c^{1/z}/(2z)$ 
and writing $c_2 = \frac{c^2}{64Az^2}$, and $c_3 = \frac{8zA\rho}{3c} $, we obtain
\[
\mathbb{P}\!\left(|\mathcal{E}(\theta')|\ge\frac{\tilde\varepsilon}{2}\right)
\le
2\exp\!\left(-\frac{c_2\varepsilon^2}{V}\right),
\]
and the constraint on $\tilde\varepsilon$ translates to
$0\le\varepsilon\le \frac{c_3 mV}{\|f_\theta\|_\infty}$.

Recall, $\Gamma$ is a $\tilde\varepsilon/(2c_0)$-net for $\parameterspc$. By Assumption~\ref{asm:class},
$|\Gamma|\le(2c_0 c_\parameterspc/\tilde\varepsilon)^{\totaldim}$.
After enlarging $c_1$ if necessary, $\log|\Gamma|\le c_1 \totaldim-\totaldim\log\varepsilon$.
Combining all estimates,
\[
\mathbb{P}\!\left(\exists\,f\in\mathcal{F}:\left|\frac{L_{\mathcal{S}}(f)}{L(f)}-1\right|\ge\varepsilon\right)
\le
2\exp\!\left(c_1 \totaldim - \totaldim\log\varepsilon - \frac{c_2\varepsilon^2}{V}\right)
\]
for all $0\le\varepsilon\le\min\!\left\{1,\,c_3 mV/\|f_\theta\|_\infty\right\}$.
\end{proof}

\subsection{Euclidean $(k,z)$-Clustering}
\label{sec:k-z-coreset}

In the ensuing section, we now apply Theorem~\ref{thm:kz-clustering-dpp} to the $(k,z)$ problem with $\mathcal{X} = \mathbb{R^d}$ to derive the following Corollary. We will then show that Theorem~\ref{thm:main-informal} holds.

\begin{corollary}[Coreset Size for Euclidean $(k,z)$-Clustering]
\label{cor:kz-coreset-size}
Given $\varepsilon, \delta \geq 0$, let \(\mathcal S\) be drawn from a  \DPP with a Hermitian kernel \(\mathbf K\) and with respect to the counting measure of a finite set \(\mathcal P=\{x_1,\dots,x_n\} \subset \mathbb{R}^d\). Let \(m=\mathbb E[|\mathcal S|]\). 
Assume that for all \(i\in\{1,\dots,n\}\), \(\mathbf K(x_i,x_i)\ge \rho\cdot m/n\) for some \(\rho>0\) not depending on \(n\). Suppose $ \sup_{\theta\in\parameterspc}\mathrm{Var}\!\left[\frac1nL_{\mathcal S}(f_\theta)\right] = \mathcal{O}(m^{-1-1/d}) $ and \(|\mathcal S|\le B\cdot m\) almost surely for some constant \(B>0\). Further  suppose for some constants, $c$ and $c'$, $m$ satisfies
\[
c'\cdot\frac{1}{\varepsilon^d} \ge m \ge c
\left(
\frac{kd + kd\log(1/\varepsilon) + \log(2/\delta)}{\varepsilon^2}
\right)^{\!\frac{d}{d+1}},
\]
Then under assumption~\ref{asm:mult error} on data, $\mathcal{S}$ is an $\varepsilon$-coreset with probability at least $1-\delta$ (over the randomness of the coreset construction algorithm).

\end{corollary}

Before proceeding with the proof argument, observe that Corollary~\ref{cor:kz-coreset-size} along with Remarks~\ref{rem:dataassumption} and ~\ref{rem:OPE_DPP_consequences1} immediately imply Theorem~\ref{thm:main-informal}.

\begin{proof}[Proof Sketch]
    We verify that the Euclidean \((k,z)\)-clustering loss class satisfies the
hypotheses of Theorem~\ref{thm:kz-clustering-dpp}. Equip the space
\(\parameterspc\) of \(k\)-center sets with the Hausdorff metric \(\dist_H\). The
nearest-center map is Lipschitz under this metric: if \(\eta^*\in\theta\) is a
closest center to \(x\), then some \(\eta'\in\theta'\) lies within
\(\dist_H(\theta,\theta')\) of \(\eta^*\), so the triangle inequality gives
\(\dist(x,\theta')-\dist(x,\theta)\le \dist_H(\theta,\theta')\); reversing the roles of
\(\theta\) and \(\theta'\) yields
$
|\dist(x,\theta)-\dist(x,\theta')|\le \dist_H(\theta,\theta').
$

Since \(f_\theta(x)^{1/z}=\dist(x,\theta)\), Assumption~\ref{asm:lipschitz}
follows. Moreover, because the candidate centers lie in a bounded subset
\(K\subset\mathbb R^d\), a Euclidean \(u\)-net of \(K\) of size
\((C_K/u)^d\) induces a \(u\)-net of \(\parameterspc\) of size
\((C_K/u)^{kd}\). Thus Assumption~\ref{asm:class} holds with
\(\totaldim=kd\). The boundedness of \(K\) also gives
\(\sup_{\theta}\|f_\theta\|_\infty<\infty\), and the admissible-range condition
of Theorem~\ref{thm:kz-clustering-dpp}, namely
\(\varepsilon\le c_3 mV/\sup_{\theta}\|f_\theta\|_\infty\), together with
the variance scaling \(V=\mathcal O(m^{-(1+1/d)})\), gives
\(\varepsilon=\mathcal O(m^{-1/d})\), or equivalently \(m=\mathcal O(\varepsilon^{-d})\).
This is precisely the upper bound \(m\le c'\cdot\varepsilon^{-d}\) stated in
the corollary, ensuring the theorem is applied within its valid regime.
Applying Theorem~\ref{thm:kz-clustering-dpp} with \(\totaldim=kd\) gives
\[
\mathbb P(\mathrm{failure})
\le
2\exp\!\left(
c_1kd-kd\log\varepsilon-\frac{c_2\varepsilon^2}{V}
\right).
\]
Requiring this probability to be at most \(\delta\), substituting
\(V=\mathcal O(m^{-(1+1/d)})\), and solving for \(m\) gives
\[
m
\gtrsim
\left(
\frac{
kd+kd\log(1/\varepsilon)+\log(2/\delta)
}{
\varepsilon^2
}
\right)^{d/(d+1)},
\]
which is the lower bound in the corollary. Combining both bounds shows that
any \(m\) in the stated range guarantees failure probability at most \(\delta\),
so \(\mathcal{S}\) is an \(\varepsilon\)-coreset with probability at least
\(1-\delta\).
\end{proof}

\paragraph{Discussions on Assumptions}
\begin{remark}[On the assumption \ref{asm:mult error}] \label{rem:dataassumption}
Recall that Assumption \ref{asm:mult error} requires
\(
L(f)/n \ge c
\)
uniformly over \(f\in \mathcal{F}\). However, since the dataset \(\mathcal{P}\) is randomly generated,
\(
L(f)=\sum_{i=1}^n f(X_i)
\)
depends on the particular realization of the data, making the verification of Assumption \ref{asm:mult error} potentially delicate. We address this issue through a technical result in the appendix; see Theorem \ref{thm:app-empirical-bound}. This result verifies Assumption \ref{asm:mult error} with high probability in the dataset $\mathcal{P}$, under mild assumptions on the data distribution \(\mu\).
\end{remark}

\begin{remark}[Variance reduction and discretized multivariate OPE]\label{rem:OPE_DPP_consequences1}
A key ingredient in the above corollary is the variance bound
\(V = \mathcal{O}(m^{-1-1/d}),\)
which is achieved by the discretized multivariate OPE kernel constructed in
\cite{bardenet2021dpp_sgd,bardenet2024small}. More precisely, this DPP satisfies
\(\var[n^{-1}L_{\mathcal{S}}(f)]=\mathcal{O}(m^{-1-1/d})\)
for any sufficiently regular test function \(f\), with high probability over the dataset
\(\mathcal{P}=\{X_1,\dots,X_n\}\), where the \(X_i\) are i.i.d. samples from a distribution supported on a \(d\)-dimensional hypercube. For a more detailed discussion of the construction of these discrete OPE kernels, we refer to
\cite{bardenet2021dpp_sgd,bardenet2024small} and the references therein.

We remark that the property
\(
V = \mathcal{O}(m^{-1-\alpha})\)
for some \(\alpha>0\), often referred to as \emph{variance reduction}, is one of the key features of DPP samplers. It has played an important role in proving theoretical guarantees that DPP-based samplers can perform better than i.i.d. sampling; see, for example, \cite{bardenet2020monte} for DPP-based quadratures, \cite{bardenet2021dpp_sgd} for DPP-based minibatches, and \cite{bardenet2024small} for DPP-based coresets. An ongoing line of work investigates this phenomenon for substantially broader classes of continuous kernels, together with mechanisms for transferring the corresponding variance-reduction guarantees to discrete kernels.

\end{remark}

\begin{remark}
    Corollary~\ref{cor:kz-coreset-size} gives a trade-off between the size $m$ of a DPP-based coreset and the accuracy rate $\varepsilon$. In particular, for an OPE-based coreset of size $m$ with $V=\mathcal{O}(m^{-1-1/d})$, one obtains the high-probability accuracy rate $\varepsilon = \mathcal{O}(m^{-1/2-1/(2d)})$, whereas independent-sampling coresets achieve $\varepsilon=\mathcal{O}(m^{-1/2})$.
\end{remark}

\section{Experiments}
\label{sec:experiments}

We evaluate our DPP-based coreset samplers against standard baselines,
including state-of-the-art near-optimal sensitivity construction by~\cite{huang2024optimal} (abbreviated HLW). Two questions drive the
evaluation: whether DPP samplers exhibit the faster error decay predicted by
theory, and whether this advantage persists on real-world clustering benchmarks.
All algorithms are implemented in Python~3 and run on an Apple MacBook with a
10-core M5 chip and 8GB of RAM, using \texttt{NumPy}, \texttt{SciPy}, and
\texttt{scikit-learn}~\cite{pedregosa2011scikit}.
The OPE feature construction uses the \texttt{MultivariateJacobiOPE} class
from \texttt{DPPy}~\cite{gautier2019dppy}.
For each method and coreset size $m$ we draw $N_{\mathrm{REP}} = 150$
independent coresets and evaluate them on a fixed family of
$N_{\mathrm{QUERY}} = 150$ random $k$-tuples of data points (as potential centers), drawn once per
dataset $k$ pair.
We report the average-case relative error over this query family;
supplementary supremum-style metrics appear in
Appendix~\ref{app:supp-benchmark}.

\noindent \textbf{Methods.}
We compare seven coreset samplers; each takes a dataset $\mathcal{P}\subset\mathbb{R}^d$
and a target size $m$ and returns a weighted subset $\mathcal{S}\subset\mathcal{P}$.
For a query center set $C$, the full data loss is
$L(C):=\frac{1}{n}\sum_{x\in\mathcal{P}} \dist^2(x,C)$, and the coreset estimate is
$L_{\mathcal S}(C):=\sum_{x\in\mathcal S}w(x)\dist^2(x,C)$, with method-specific
weights $w$ chosen to estimate the average loss $L(C)$. For independent
importance sampling with probabilities $p_x$, we use the unbiased weight
$w(x)=1/(mnp_x)$ per sampled copy; for a discrete DPP with marginal kernel $K$,
we use $w(x)=1/(nK(x,x))$.

\noindent\textbf{Classical baselines.}
Uniform draws $m$ points uniformly with replacement in $O(m)$ time. Sensitivity ~\cite{langberg2010universal} is the importance-sampling coreset
using the practical sensitivity upper bounds of Bachem~et~al.~\cite{bachem2017practical}
from a $D^2$-weighted initialisation~\cite{ArthurV07}, in time $O(nk+nm)$.
HLW~\cite{huang2024optimal} is the near-optimal construction by Huang, Li,
and Wu, based on a bicriteria $D^2$-approximation, ring decomposition, and per-ring
sensitivity bounds. Uniform has the usual Monte Carlo $m^{-1/2}$ error rate for a fixed query family,
but is not a worst-case clustering coreset. Sensitivity sampling gives the
standard sensitivity-based $\widetilde O(S\varepsilon^{-2})$ coreset size, where
$S$ is the total sensitivity upper bound. For $k$-means in Euclidean space,
HLW gives the near-optimal worst-case bound
$\widetilde O(k^{3/2}\varepsilon^{-2})$.

\noindent\textbf{Negative-dependence (DPP) samplers.}
Stratified ~\cite{bardenet2024small} is a simple projection DPP for
$\mathcal{P}\subset[-1,1]^d$: it partitions $[-1,1]^d$ into a grid of $m$ equal bins
and draws one point uniformly from $\mathcal{P}$ in each occupied bin, running in
$O(nm)$. Its limitation is immediate: it relies on the data being sufficiently
well-spread so that many bins are occupied, which is unlikely for non-uniform or
high-dimensional datasets.
DPP-Haar is a dyadic-cell Haar-wavelet projection DPP: it partitions
$[-1,1]^d$ into $2^{jd}$ cells with $j\approx\log_2(m)/d$, samples one representative
per selected non-empty cell, and reweights by cell and in-cell selection probabilities.
DPP-OPE ~\cite{bardenet2024small} (Example~5 therein) is the discretised
orthogonal polynomial ensemble, implemented via \texttt{MultivariateJacobiOPE} from
\texttt{DPPy}~\cite{gautier2019dppy}; its Jacobi reference-measure parameters are fitted
coordinatewise from marginal moments with a $10\%$ uniform fallback, it runs in
$O(nm^2)$, and it is restricted to $d\le 6$.
DPP-Vandermonde ~\cite{tremblay2019determinantal} (Algorithm~2) is a projection
DPP built from the monomial Vandermonde matrix with graded total-degree multi-indices and
an economy QR factorisation in $O(nm^2)$; it requires no density estimation and applies
up to $d\le 16$. We call DPP-OPE and DPP-Vandermonde the polynomial DPPs.

\noindent \textbf{Datasets.}
All datasets are normalized coordinatewise to $[-1,1]^d$.
The synthetic collection comprises six datasets: Uniform
($n=982$, $d=2$), Trimodal ($n=981$, $d=2$), Gaussian Blobs 2D ($n=4,800$,$d=2$), Gaussian Blobs 3D ($n=4,702$, $d=3$), Anisotropic Gaussians ($n=2,880$, $d=2$),
and BIRCH Benchmark ($n=9,600$, $d=2$).
Uniform 2D is drawn uniformly from $[-1,1]^2$. Trimodal 2D comprises three
isotropic Gaussian clusters with standard deviations $0.04$, $0.15$, and
$0.40$. Gaussian Blobs 2D and Gaussian Blobs 3D are five well-separated isotropic Gaussian blobs (std $0.3$)
generated with \texttt{make\_blobs}. Anisotropic Gaussians 2D contains three
anisotropic Gaussian clusters, two of which are elongated along rotated axes.
BIRCH Benchmark is a subsampled instance of 100 Gaussian blobs, used as a
low-dimensional many-cluster benchmark~\cite{zhang1997birch}. All these synthetic datasets are sub-Gaussian and thus satisfy our required assumptions with high probability.

For real-world evaluation, we use three datasets that collectively probe
different dimensionalities and data geometries.
3D Road Network ($n=9,424$, $d=3$) is the
\emph{3D Road Network (North Jutland, Denmark)} dataset from the
{UCI} Machine Learning Repository~\cite{kaul2013roadnetwork}.
UCI Adult (PCA-2) ($n=9,338$, $d=2$) is the
{UCI} Adult dataset~\cite{becker1996adult}, represented by its six numeric attributes (age,
fnlwgt, education-num, capital-gain/loss, hours-per-week) and then projected
onto the top two principal components via PCA~\cite{pedregosa2011scikit};
the $d=2$ setting enables all DPP samplers and gives a clean benchmark
for the polynomial constructions on real economic data.
Twitter GPS ($n=9,644$, $d=2$)~\cite{chan2018twitter}
consists of geotagged longitude-latitude coordinates; its strong
spatial clustering structure makes it particularly well-suited to spatial DPP methods.

\noindent \textbf{Performance metric.}
We report the \emph{mean relative error} $\overline{\mathrm{RelErr}}(m)$,
defined as the average over $N_{\mathrm{REP}}=150$ independent coreset draws
of the per-draw average relative error:
\[
  \overline{\mathrm{RelErr}}(m)
  \;:=\;
  \frac{1}{N_{\mathrm{REP}}}\sum_{r=1}^{N_{\mathrm{REP}}}
  \frac{1}{N_{\mathrm{QUERY}}}
  \sum_{j=1}^{N_{\mathrm{QUERY}}}
  \frac{\lvert L_{\mathcal{S}^{(r)}}(C_j)-L(C_j)\rvert}{L(C_j)},
\]
where $\mathcal{S}^{(r)}$ denotes the $r$-th independently drawn coreset and the
query centers $C_1,\dots,C_{N_{\mathrm{QUERY}}}$ are fixed random $k$-tuples
drawn once per dataset-$k$ pair.
Coreset sizes range over $m\in\{1^2,2^2,\ldots,29^2\}$ ($29$ values),
clipped to $m\le 300$ for datasets with $n<5{,}000$ and to $m\le n/3$
in all cases. We fix $k=5$ throughout; supplementary variance and
supremum-style plots appear in Appendix~\ref{app:supp-benchmark}.

\noindent \textbf{Research questions.}
\begin{enumerate}[leftmargin=*]
  \item \textbf{(RQ1) Faster error decay.}
        Do DPP-based samplers achieve a steeper log--log decay of
        $\overline{\mathrm{RelErr}}(m)$, improving upon the $m^{-1/2}$ rate of
        i.i.d.\ methods, as predicted by Theorem~\ref{thm:kz-clustering-dpp}?
  \item \textbf{(RQ2) Smaller coreset on real data.}
        Do DPP methods attain a given target error with a smaller coreset than
        i.i.d.\ baselines on real-world benchmarks?
\end{enumerate}

\noindent \textbf{Results: RQ1}
Across all six synthetic datasets (Figure~\ref{fig:lg_scle_synth}), DPP-OPE and
DPP-Vandermonde show a visibly steeper log--log decay than the i.i.d.\ baselines
Uniform, Sensitivity, and HLW~\cite{huang2024optimal}. On the $d=2$ benchmarks,
the empirical slope is consistent with the predicted $m^{-(1/2+1/4)}$ rate from
Theorem~\ref{thm:kz-clustering-dpp}, whereas on Gaussian Blobs 3D the gap narrows
to the expected $m^{-(1/2+1/6)}$ regime. DPP-Haar matches or exceeds DPP-OPE and DPP-Vandermonde on the more spatially regular datasets. Stratified improves over
Uniform when the data are well spread, but degrades on clustered or strongly
non-uniform datasets, exactly as expected from its grid-based construction
~\cite{bardenet2024small}.

\begin{figure}[htbp]
  \centering
  \begin{subfigure}[b]{0.25\textwidth}
    \includegraphics[width=\textwidth]{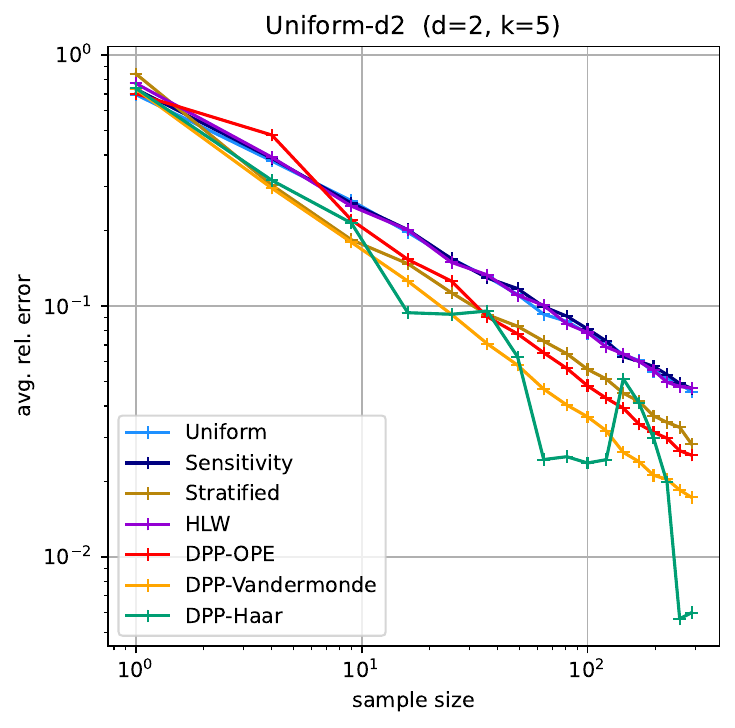}
    \caption{\textbf{Uniform}}
  \end{subfigure}
  \hfill
  \begin{subfigure}[b]{0.25\textwidth}
    \includegraphics[width=\textwidth]{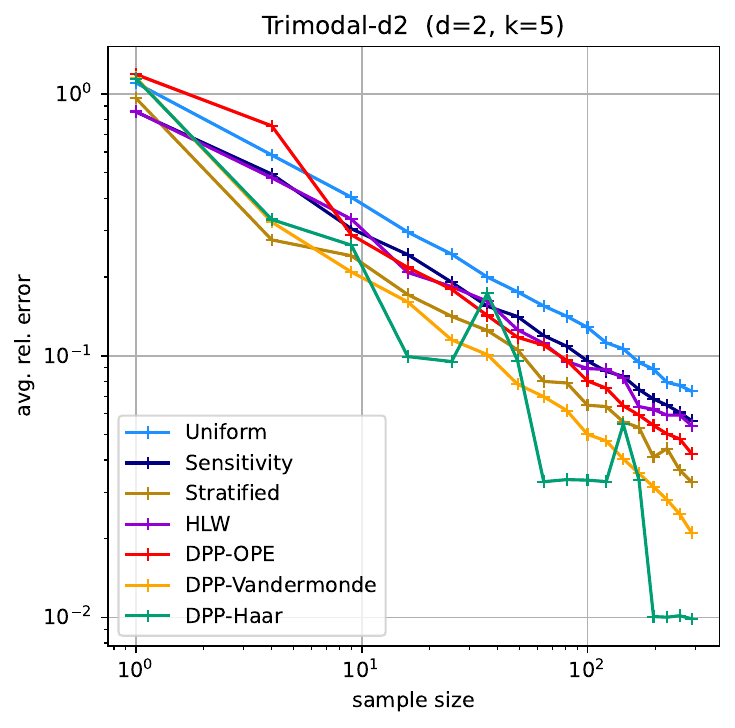}
    \caption{\textbf{Trimodal}}
  \end{subfigure}
  \hfill
  \begin{subfigure}[b]{0.25\textwidth}
    \includegraphics[width=\textwidth]{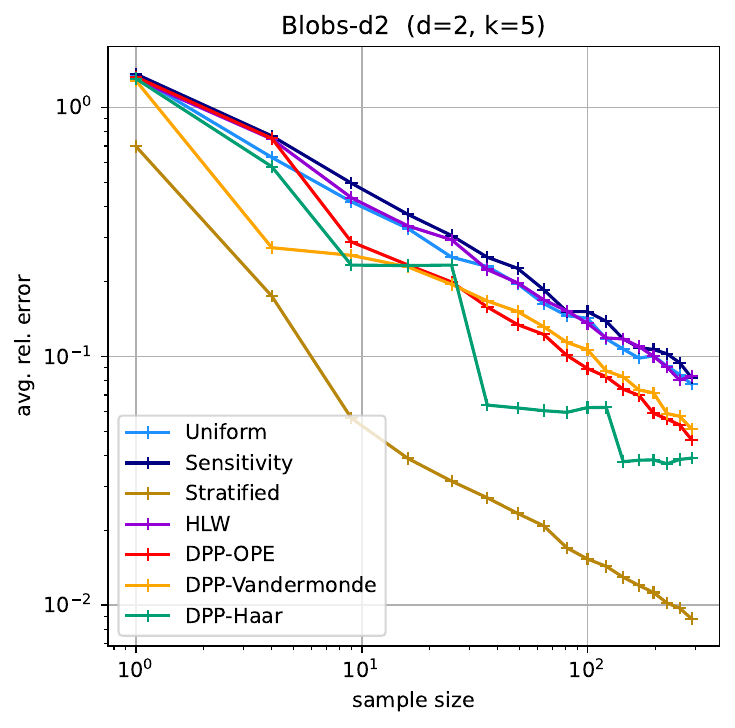}
    \caption{\textbf{Gaussian Blobs 2D}}
  \end{subfigure}
  \\[4pt]
  \begin{subfigure}[b]{0.25\textwidth}
    \includegraphics[width=\textwidth]{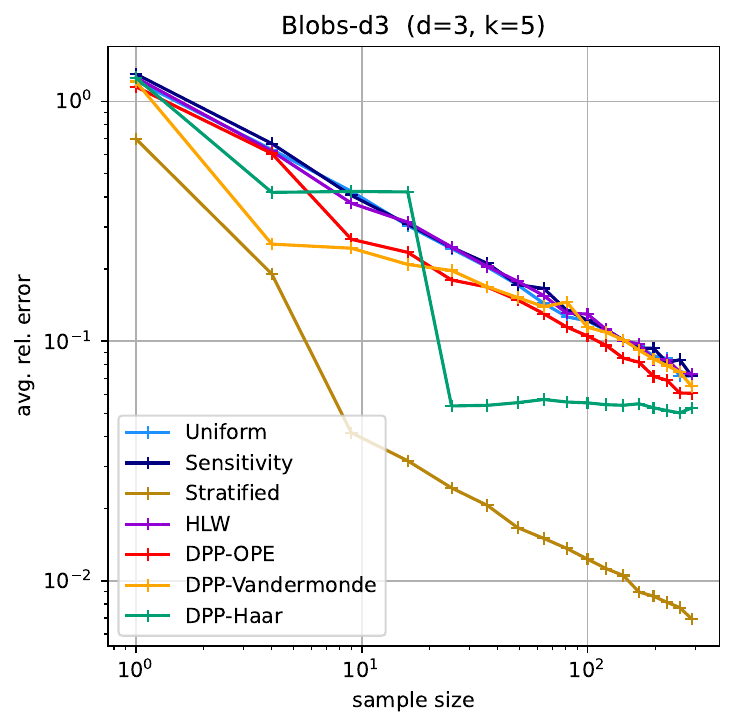}
    \caption{\textbf{Gaussian Blobs 3D}}
  \end{subfigure}
  \hfill
  \begin{subfigure}[b]{0.25\textwidth}
    \includegraphics[width=\textwidth]{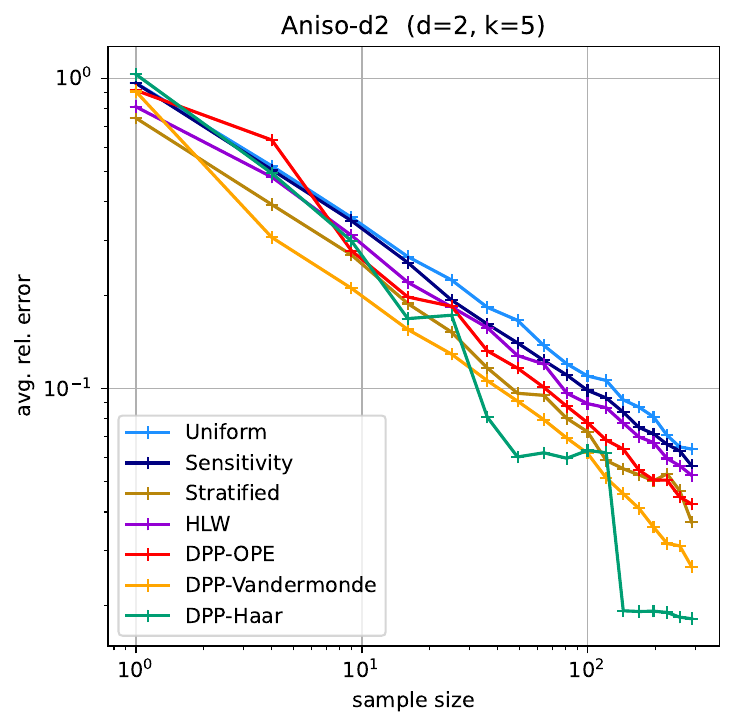}
    \caption{\textbf{Anisotropic Gaussian}}
  \end{subfigure}
  \hfill
  \begin{subfigure}[b]{0.25\textwidth}
    \includegraphics[width=\textwidth]{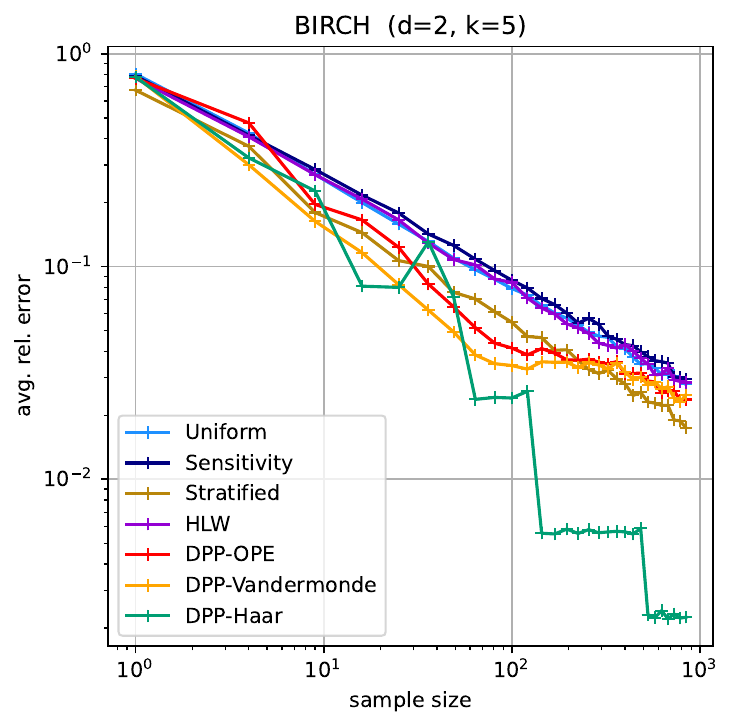}
    \caption{\textbf{BIRCH}}
  \end{subfigure}
  \caption{$\overline{\mathrm{RelErr}}(m)$ vs.\ $m$ (log--log, $k=5$), synthetic datasets.}
  \label{fig:lg_scle_synth}
\end{figure}

\noindent \textbf{Results: RQ2}
Figure~\ref{fig:lg_scl_real} shows the same qualitative pattern on real data:
for moderate and large $m$, at least one DPP sampler reaches a given error with
fewer points than the independent baselines. The gains are smaller on 3D Road
Network, consistent with the weaker $d=3$ rate. On UCI Adult (PCA-2), DPP-OPE
and DPP-Vandermonde are strongest, while on Twitter GPS, DPP-Haar performs best,
suggesting that its multiscale spatial partition is well matched to geographic
data. Stratified is less reliable on these non-uniform datasets, in line with
its occupancy-based limitations~\cite{bardenet2024small}.

\begin{figure}[htbp]
  \centering
  \begin{subfigure}[b]{0.25\textwidth}
    \includegraphics[width=\textwidth]{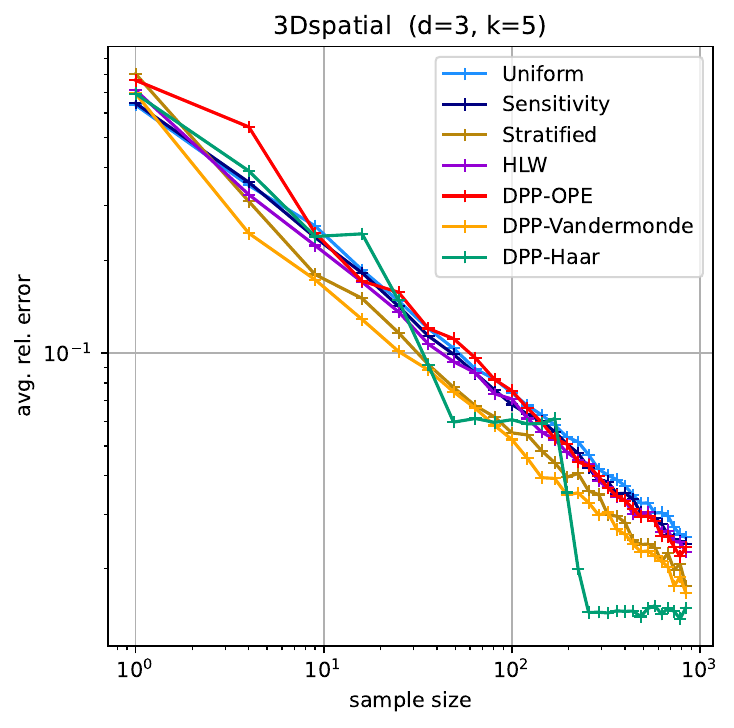}
    \caption{\textbf{3D Road Network}}
  \end{subfigure}
  \hfill
  \begin{subfigure}[b]{0.25\textwidth}
    \includegraphics[width=\textwidth]{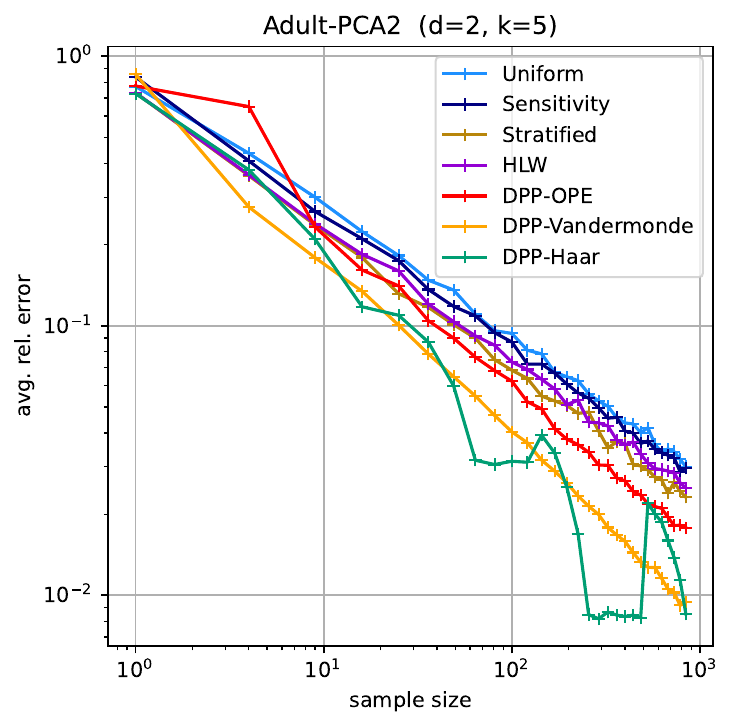}
    \caption{\textbf{UCI Adult (PCA-2)}}
  \end{subfigure}
  \hfill
  \begin{subfigure}[b]{0.25\textwidth}
    \includegraphics[width=\textwidth]{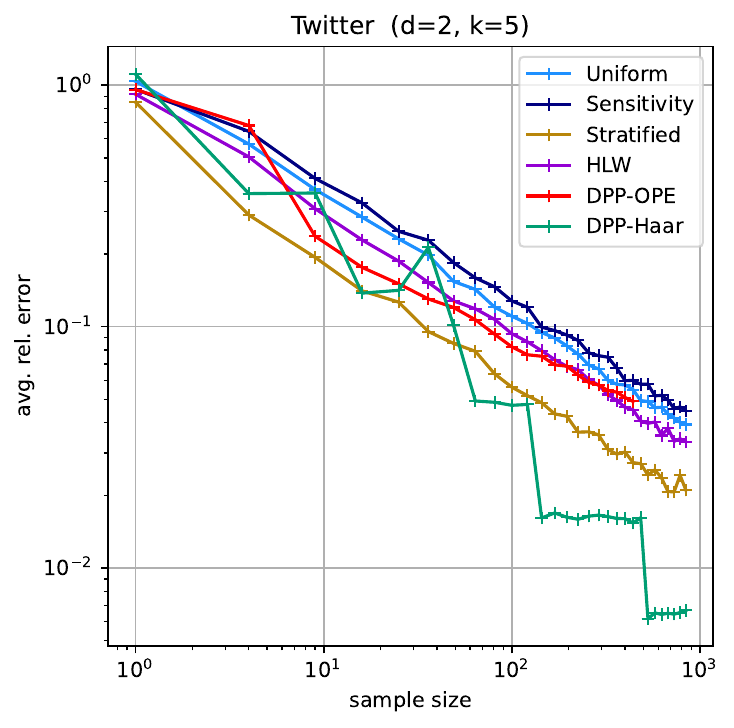}
    \caption{\textbf{Twitter GPS}}
  \end{subfigure}
  \caption{$\overline{\mathrm{RelErr}}(m)$ vs.\ $m$ (log--log, $k=5$), real-world datasets.}
  \label{fig:lg_scl_real}
\end{figure}

\section{Conclusion and Future Work}
In this paper, we study the $(k,z)$-clustering problem from a beyond-worst-case perspective and show that substantially smaller coresets are possible under natural assumptions on the input distribution. 
To achieve this, we introduce determinantal sampling 
and use it to obtain stronger guarantees than those available through standard independent sampling techniques. Beyond its theoretical significance, we show the efficacy of our method by performing experiments on various standard datasets. Overall, our results provide further evidence that realistic structural assumptions can meaningfully improve the landscape of coreset construction for clustering.

It would be interesting to relax or further generalize the distributional assumptions used in our analysis, and to better understand the weakest conditions under which one can still beat the worst-case lower bounds. Another important open question is whether similar improvements can be achieved for broader classes of clustering objectives, more general metric spaces, or in high-dimensional settings without relying on dimensionality reduction. Finally, it would be valuable to explore whether the determinantal-sampling framework can be extended beyond clustering to other coreset problems in regression, subspace approximation, and related summarization tasks.

\clearpage

\bibliographystyle{alpha}
\bibliography{biblio}

\newpage

\appendix

\noindent

\section{Technical results related to Theorem~\ref{thm:kz-clustering-dpp}}
\label{app:proofs-section3}
\subsection{A simple algebraic identity}
\label{app:alg_ineq}

\begin{proposition}\label{prop:alg_ineq_1}
  For $\theta\in\parameterspc$, define
\[
Y(\theta):=\frac{L_{\mathcal{S}}(f_\theta)}{n},\qquad
Z(\theta):=\frac{L(f_\theta)}{n},\qquad
\mathcal{E}(\theta):=Y(\theta)^{1/z}-Z(\theta)^{1/z},
\]
and let $\tilde\varepsilon:=\varepsilon c^{1/z}/(2z)$. Then  
    \[
\mathbb{P}\!\left(\exists\,f\in\mathcal{F}:\left|\frac{L_{\mathcal{S}}(f)}{L(f)}-1\right|\ge\varepsilon\right)
\le
\mathbb{P}\!\left(\sup_{\theta\in\parameterspc}|\mathcal{E}(\theta)|\ge\tilde\varepsilon\right)
\le
\mathbb{P}\!\left(\frac{1}{n}\bigl|L_{\mathcal{S}}(f_{\theta})-L(f_{\theta})\bigr|
\ge
\frac{\tilde\varepsilon}{2}\,c^{(z-1)/z}\right).
\]

\end{proposition}

\begin{proof}
If $\sup_{\theta\in\parameterspc}|\mathcal{E}(\theta)|\le\tilde\varepsilon$, then
\[
\frac{Y(\theta)}{Z(\theta)}-1
=
\mathcal{E}(\theta)\,Z(\theta)^{-1/z}
\sum_{i=0}^{z-1}\left(\frac{Y(\theta)}{Z(\theta)}\right)^{i/z}.
\]
By simplifying the above, since $Z(\theta)\ge c$ by assumption~\ref{asm:mult error},
\[
\frac{Y(\theta)^{1/z}}{Z(\theta)^{1/z}}
=
1+\frac{\mathcal{E}(\theta)}{Z(\theta)^{1/z}}
\le
1+\frac{\tilde\varepsilon}{c^{1/z}}
=
1+\frac{\varepsilon}{2z}
\le e^{\varepsilon/(2z)}.
\]
Hence, for every $0\le i\le z-1$,
\[
\left(\frac{Y(\theta)}{Z(\theta)}\right)^{i/z}
=
\left(\frac{Y(\theta)^{1/z}}{Z(\theta)^{1/z}}\right)^i
\le e^{i\varepsilon/(2z)}
\le e^{\varepsilon/2}
< 2,
\]
provided $0\le\varepsilon\le 1$.
Therefore the sum is bounded by $2z$, and thus
\[
\left|\frac{Y(\theta)}{Z(\theta)}-1\right|
\le \tilde\varepsilon\, c^{-1/z}\cdot 2z
= \varepsilon.
\]

On the other hand since
\[
Y(\theta')-Z(\theta')
=
\mathcal{E}(\theta)\sum_{i=0}^{z-1}Y(\theta)^{i/z}Z(\theta)^{(z-1-i)/z},
\]
and all terms are nonnegative, keeping only the $i=0$ term gives
\[
|Y(\theta)-Z(\theta)|
\ge
|\mathcal{E}(\theta)|\,Z(\theta)^{(z-1)/z}
\ge
|\mathcal{E}(\theta)|\,c^{(z-1)/z}.
\]
Therefore
\begin{equation}
\mathbb{P}\!\left(|\mathcal{E}(\theta)|\ge\frac{\tilde\varepsilon}{2}\right)
\le
\mathbb{P}\left(\bigl|L_{\mathcal{S}}(f_{\theta})-L(f_{\theta})\bigr|
\ge
\frac{\tilde\varepsilon}{2}\,c^{(z-1)/z}\right).
\end{equation}
\end{proof}

\subsection{Concentration bounds of DPP}\label{app:conc_bound}
The following is a restatement of Theorem~1 in \cite{bardenet2024small} stated here for the reader's convenience:
\begin{theorem}[Theorem~1 in \cite{bardenet2024small}]\label{thm:thm1_bardenet}
Let $\mathcal{S}$ be a DPP on a Polish space $\mathcal{X}$ with reference measure $\mu$ and Hermitian kernel $K$. Then there is a universal constant $A$ such that for any bounded test function $\phi : \mathcal{X} \rightarrow \mathbb{R}$, and $\forall\ 0 \leq \epsilon \leq 2A \frac{\mathrm{Var}\left[\sum_{x \in \mathcal{S}} \phi(x)\right]}{3 \|\phi\|_{\infty}}$
 we have
\[
\mathbb{P}\left(\left|\sum_{x \in \mathcal{S}} \phi(x) - \mathbb{E}\left[\sum_{x \in \mathcal{S}} \phi(x)\right]\right| \geq \epsilon \right)
\leq
2 \exp\left({-\frac{\epsilon^2}{4A \mathrm{Var}\left[\sum_{x \in \mathcal{S}} \phi(x)\right]}}\right).
\]
\end{theorem}

\section{Technical results related to Corollary~\ref{cor:kz-coreset-size}}

\subsection{Proof of Corollary~\ref{cor:kz-coreset-size}}
\label{app:prf_coro_corsize}

\begin{proof}
Let $K\subset\mathbb{R}^d$ be a bounded set containing all candidate centers,
let $\parameterspc$ be the family of unordered $k$-point subsets of $K$, equipped with
the Hausdorff metric $\dist_H$, and set $f_\theta(x)=\dist(x,\theta)^z$.
\[
\dist_H(\theta,\theta')
:=
\max\!\left\{
  \max_{\eta\in\theta}\min_{\eta'\in\theta'} \dist(\eta,\eta'),\;
  \max_{\eta'\in\theta'}\min_{\eta\in\theta} \dist(\eta,\eta')
\right\}.
\]
Then, for every $x\in\mathcal{X}$ and every $\theta,\theta'\in\parameterspc$,
\[
\bigl|\dist(x,\theta)-\dist(x,\theta')\bigr|\le \dist_H(\theta,\theta').
\]
Fix $x\in\mathcal{X}$ and $\theta,\theta'\in\parameterspc$.
Let $\eta^*\in\theta$ satisfy $\dist(x,\theta)=\dist(x,\eta^*)$.
By the definition of the Hausdorff metric, there exists $\eta'\in\theta'$ such
that $\dist(\eta^*,\eta')\le \dist_H(\theta,\theta')$.
Therefore
\[
\begin{aligned}
\dist(x,\theta')-\dist(x,\theta)
&\le \dist(x,\eta')-\dist(x,\eta^*) \\
&\le \dist(\eta^*,\eta') \\
&\le \dist_H(\theta,\theta'),
\end{aligned}
\]
where we used the reverse triangle inequality in the second step.
Interchanging the roles of $\theta$ and $\theta'$ gives
$\dist(x,\theta)-\dist(x,\theta')\le \dist_H(\theta,\theta')$.
Hence $\bigl|\dist(x,\theta)-\dist(x,\theta')\bigr|\le \dist_H(\theta,\theta')$.
Consequently, $f_\theta(x):=\dist(x,\theta)^z$, hence
\[
\bigl|f_\theta(x)^{1/z}-f_{\theta'}(x)^{1/z}\bigr|
= \bigl|\dist(x,\theta)-\dist(x,\theta')\bigr|
\le \dist_H(\theta,\theta').
\]
Assumption~\ref{asm:lipschitz} holds with $\dist_\parameterspc=\dist_H$.
Since $K$ is bounded in $\mathbb{R}^d$, there exists $C_K>0$ such that
$\mathcal{N}(K,u,\|\cdot\|)\le(C_K/u)^d$ for $0<u\le 1$.
Replacing each center of $\theta\in\parameterspc$ by its nearest point in a $u$-net
$N_u\subseteq K$ yields $\hat\theta$ with $\dist_H(\theta,\hat\theta)\le u$, so
\[
\mathcal{N}(\parameterspc,u,\dist_H)\le|N_u|^k\le\left(\frac{C_K}{u}\right)^{kd}.
\]
Hence Assumption~\ref{asm:class} holds with $\totaldim=kd$.
We next verify the upper bound on $\varepsilon$ required by
Theorem~\ref{thm:kz-clustering-dpp}.
Since $K$ is bounded, we have
\[
M := \sup_{\theta\in\parameterspc}\|f_\theta\|_\infty
= \sup_{\theta,x} \dist(x,\theta)^z
\le \mathrm{diam}(K)^z
< \infty.
\]
The theorem applies provided $\varepsilon \le c_3 m V / M$.
Substituting $V \le c_V m^{-(1+1/d)}$ gives
\[
\varepsilon \;\lesssim\; m^{-1/d},
\]
which is exactly the condition $\varepsilon=\mathcal{O}(m^{-1/d})$ stated in the
corollary. Hence all hypotheses of Theorem~\ref{thm:kz-clustering-dpp} are satisfied.
Applying Theorem~\ref{thm:kz-clustering-dpp},
\[
\mathbb{P}(\text{failure})
\le
2\exp\!\left(c_1 kd - kd\log\varepsilon - \frac{c_2\varepsilon^2}{V}\right).
\]
For this probability to be at most $\delta$, it suffices to have
\[
c_1 kd + kd\log\!\left(\tfrac{1}{\varepsilon}\right) + \log\!\left(\tfrac{2}{\delta}\right)
\le
\frac{c_2\varepsilon^2}{V}.
\]
Using $V\le c_V m^{-(1+1/d)}$ and rearranging,
\[
m^{\frac{d+1}{d}}
\ge
\frac{c_V}{c_2\varepsilon^2}
\!\left[c_1 kd + kd\log\!\left(\tfrac{1}{\varepsilon}\right) + \log\!\left(\tfrac{2}{\delta}\right)\right],
\]
and therefore
\[
m
\ge
\left(
\frac{c_V}{c_2\varepsilon^2}
\!\left[C_1 kd + kd\log\!\left(\tfrac{1}{\varepsilon}\right) + \log\!\left(\tfrac{2}{\delta}\right)\right]
\right)^{\!\frac{d}{d+1}}.
\]
The $\tilde{\mathcal{O}}$ form follows by absorbing all constants and logarithmic factors.
\end{proof}

\subsection{Discussion on Assumption~\ref{asm:mult error}} \label{app:mult_error}
\begin{definition}[Norm-tail subgaussian random vector]
Let \(X\in\mathbb R^d\) have finite mean. We say that \(X\) is
\(\sigma\)-subgaussian if
\[
    \mathbb P\left(
        \|X-\mathbb E[X]\|_2>t
    \right)
    \le
    2\exp\left(-\frac{t^2}{2\sigma^2}\right)
    \qquad
    \text{for all }t\ge 0.
\]
\end{definition}

\begin{theorem}[Uniform empirical lower bound for \(k\)-point \(z\)-cost] \label{thm:app-empirical-bound}
Let \(X_1,\ldots,X_n\) be i.i.d. samples from a $\sigma$-subgaussian
probability measure
\(\mu\) on \(\mathbb R^d\). 
Let \(k\ge 1\) and \(z\ge 1\), we define the population \(k\)-point \(z\)-cost of $\mu$ by
\[
    q_{k,z}(\mu)
    :=
    \inf_{1\le |\mathcal C|\le k}
    \mathbb E_{X\sim\mu}\bigl[\dist(X,\mathcal C)^z\bigr].
\]
Assume that
\(
    q:=q_{k,z}(\mu)>0,
\)
then there exist explicit constants
\[
    c=c(z,q,\sigma)>0
    \qquad\text{and}\qquad
    n_0=n_0(k,d,z,q,\sigma)\ge 1
\]
such that for every \(n\ge n_0\), we have
\[
    \mathbb P\left(
        \inf_{1\le |\mathcal C|\le k}
        \frac1n\sum_{i=1}^n \dist(X_i,\mathcal C)^z
        \ge
        \frac q2
    \right)
    \ge
    1-e^{-cn}.
\]
\end{theorem}

\begin{proof}
Let $X\sim \mu$ and
\(
    V:=\|X-\mathbb E[X]\|_2.
\)
By the subgaussian assumption,
\[
    \mathbb P(V>t)
    \le
    2\exp\left(-\frac{t^2}{2\sigma^2}\right),
    \qquad
    \text{for all }t\ge 0.
\]
Using the tail-integral formula,
\(
    \mathbb E[V^{2z}]
    =
    \int_0^\infty
        2z\,t^{2z-1}\mathbb P(V>t)
    \,dt,
\)
we get
\[
    \mathbb E[V^{2z}]
    \le
    2\int_0^\infty
        2z\,t^{2z-1}
        \exp\left(-\frac{t^2}{2\sigma^2}\right)
    \,dt
    =
    2(2\sigma^2)^z\Gamma(z+1) =:  A_{z,\sigma}.
\]

Let
\(
    R
    :=
    \max\{
        q^{1/z},
        (
            \frac{4\cdot 3^z A_{z,\sigma}}{q}
        )^{1/z}
    \}.
\)
Then \(R\ge q^{1/z}\), and by Markov's inequality,
\[
    \mathbb E\left[
        V^z\mathbf 1_{\{V>R\}}
    \right]
    \le
    \frac{\mathbb E[V^{2z}]}{R^z}
    \le
    \frac{A_{z,\sigma}}{R^z} \cdot
\]
Therefore, by the choice of \(R\),
\[
    3^z
    \mathbb E\left[
        V^z\mathbf 1_{\{V>R\}}
    \right]
    \le
    \frac{q}{4} \cdot
\]

Now define the compact set $K_R$ to be the ball of radius $R$ centered at $\mathbb{E}[X]$. And define compact set $K_R^+$ to be the ball of radius $R+q^{1/z}$ centered at $\mathbb{E}[X]$.
Note that if
\(x\in K_R\) and \(\theta\notin K_R^+\), then
\[
    \|x-\theta\|_2
    \ge
    \|\theta-\mathbb E[X]\|_2-\|x-\mathbb E[X]\|_2
    >
    R+q^{1/z}-R
    =
    q^{1/z}.
\]
Thus any in-ball sample \(x\in K_R\) whose closest center lies outside
\(K_R^+\) already contributes more than \(q\) to the \(z\)-cost.

Let $\textbf{C}$ be a \(k\)-tuple
\(
    \textbf{C}=(\theta_1,\ldots,\theta_k)\in (K_R^+)^k,
\)
We note that repeated centers are allowed, so \(k\)-tuples parametrize all center
sets of size at most \(k\). Then we define the truncated loss
\[
    g_{\textbf{C}}(x)
    :=
    \dist(x,\textbf{C})^z
    \mathbf 1_{\{x\in K_R\}}.
\]

We claim that
\begin{equation} \label{eq:exp-g-theta}
    \mathbb{E}[g_{\textbf{C}}(X)] \ge \frac{3q}{4} \qquad\text{for every } \textbf{C} \in (K_R^+)^k.
\end{equation} 
Indeed, for
\(\textbf{C}\in(K_R^+)^k\),
\[
    \mathbb E[ g_\textbf{C}(X)]
    =
    \mathbb E [\dist(X,\mathcal C_\Theta)^z]
    -
    \mathbb E\left[
        \dist(X,\textbf{C})^z
        \mathbf 1_{\{X\notin K_R\}}
    \right]
    \ge q -
    \mathbb E\left[
        \dist(X,\textbf{C})^z
        \mathbf 1_{\{X\notin K_R\}}
    \right]
    ,
\]
where we used
\(
    \mathbb E [\dist(X,\textbf{C})^z ] \ge q
\)
which follows by the definition of \(q\).
Now we observe that
on the event \(\{X\notin K_R\}\), we have \(V = \|X- \mathbb E[X]\|>R\). For every center
\(\theta_j\in K_R^+\), we have
\(
    \|\theta_j-\mathbb E[X]\|_2
    \le
    R+q^{1/z}
    \le
    2R.
\)
Hence, on \(\{X\notin K_R\}\), for $C_j \in K_R^+$
\[
    \|X-\theta_j\|_2
    \le
    \|X-\mathbb E[X]\|_2+\|\theta_j-\mathbb E[X]\|_2
    \le
    V+2R
    \le
    3V.
\]
Therefore,
\[
    \dist(X,\textbf{C})^z
    \mathbf 1_{\{X\notin K_R\}}
    \le
    3^z V^z\mathbf 1_{\{V>R\}}.
\]
Taking expectation gives
\[
    \mathbb E\left[
        \dist(X,\textbf{C})^z
        \mathbf 1_{\{X\notin K_R\}}
    \right]
    \le \mathbb{E}[3^z V^z\mathbf 1_{\{V>R\}}] \le 
    \frac {q}{4} \cdot
\]
Therefore
\[
    \mathbb E [g_\textbf{C}(X)]
    \ge
    q-\frac q4
    =
    \frac{3q}{4}
    \qquad
    \text{for every }\textbf{C}\in(K_R^+)^k.
\]

For \(x\in K_R\) and \(\textbf{C}\in(K_R^+)^k\),
\[
    \dist(x,\textbf{C})
    \le
    R+R+q^{1/z}
    \le
    3R.
\]
Thus
\[
    0\le g_\textbf{C}(x)\le (3R)^z
    \qquad
    \text{for all }x\in\mathbb R^d.
\]

As in the proof of Corollary~\ref{cor:kz-coreset-size}, for every \(x\in K_R\),
\[
    \abs{\dist(x,\textbf{C}) -  \dist(x,\textbf{C}')} \le \dist_H(\textbf{C},\textbf{C}')    
\]
Consider the function \(\varphi:t\mapsto t^z\) on \([0,3R]\), we have $|\varphi'(t)| \le z(3R)^{z-1}:=L_R $. Thus,
we have
\[
    |g_\textbf{C}(x)-g_{\textbf{C}'}(x)|
    \le
    L_R \dist_H(\textbf{C},\textbf{C}')
    \qquad
    \text{for all }x\in\mathbb R^d.
\]

Set
\(
    \delta:=q/(16L_R).
\)
Let \(\mathcal N_\delta\) be a \(\delta\)-net of \(K_R^+\). We may take
\(
    |\mathcal N_\delta|
    \le
    \left(1+\frac{4R}{\delta}\right)^d.
\)
Therefore \(\mathcal N_\delta^k\) is a \(\delta\)-net of
\((K_R^+)^k\), and
\(
    |\mathcal N_\delta^k|
    \le
    \left(1+\frac{4R}{\delta}\right)^{kd}.
\)
For each fixed \(\textbf{C}\in\mathcal N_\delta^k\), Hoeffding's inequality
gives
\[
    \mathbb P\left(
        \frac1n\sum_{i=1}^n g_\textbf{C}(X_i)
        <
        \mathbb E[g_\textbf{C}(X)] -\frac q8
    \right)
    \le
    \exp\left(
        -\frac{nq^2}{32(3R)^{2z}}
    \right).
\]
Taking a union bound over \(\mathcal N_\delta^k\), there is an event $\mathcal{E}$ with probability at
least
\[
    1-
    \left(1+\frac{4R}{\delta}\right)^{kd}
    \exp\left(
        -\frac{nq^2}{32(3R)^{2z}}
    \right),
\]
such that on $\mathcal{E}$ we have
\[
    \frac1n\sum_{i=1}^n g_\textbf{C}(X_i)
    \ge
    \mathbb E [g_\textbf{C}(X)]-\frac q8, \quad \forall \Theta\in\mathcal N_\delta^k.
\]
On this event, fix arbitrary
\(
    \textbf{C}\in(K_R^+)^k
\)
and choose \(\textbf{C}'\in\mathcal N_\delta^k\) such that
\(
   \dist_H(\textbf{C},\textbf{C}') \le \delta.
\)
By the Lipschitz bound, we have
\[
    \frac1n\sum_{i=1}^n g_\textbf{C}(X_i)
    \ge
    \frac1n\sum_{i=1}^n g_{\textbf{C}'}(X_i)
    -
    L_R\delta,
\]
and
\(
    \mathbb E [g_{\textbf{C}'}(X)]
    \ge
    \mathbb E[g_\textbf{C}(X)]-L_R\delta.
\)
Therefore,
\[
    \frac1n\sum_{i=1}^n g_\textbf{C}(X_i)
    \ge
    \mathbb E [g_\textbf{C}(X)]
    -
    \frac q8
    -
    2L_R\delta.
\]
Since \(\delta=q/(16L_R)\),
\(
    2L_R\delta=q/8.
\)
Thus
\[
    \frac1n\sum_{i=1}^n g_\textbf{C}(X_i)
    \ge
    \mathbb E [g_\textbf{C}(X)]-\frac q4.
\]
By \eqref{eq:exp-g-theta}, we have \(\mathbb E [g_\textbf{C}(X)]\ge 3q/4\). Thus, on the event $\mathcal{E}$:
\begin{equation} \label{eq:gC_positive_compact}
    \frac1n\sum_{i=1}^n g_\textbf{C}(X_i)
    \ge
    \frac q2
    \qquad
    \text{for every }\textbf{C}\in(K_R^+)^k.
\end{equation}

It remains to pass from centers in \(K_R^+\) to arbitrary centers in
\(\mathbb R^d\). Let
\(
    \Pi:\mathbb R^d\to K_R^+
\)
be Euclidean projection onto the closed convex set \(K_R^+\). Since
\(K_R\subseteq K_R^+\), projection can only decrease distance from
points \(x\in K_R\), i.e.
\[
    \|x-\Pi(\theta)\|_2
    \le
    \|x-\theta\|_2
    \qquad
    \text{for all }x\in K_R,\ \theta\in\mathbb R^d.
\]
Take arbitrary
\(
    \textbf{C}=(\theta_1,\ldots,\theta_k)\in(\mathbb R^d)^k,
\)
and define
\(
    \Pi\textbf{C}
    :=
    (\Pi(\theta_1),\ldots,\Pi(\theta_k))\in(K_R^+)^k.
\)
For \(x\in K_R\), we then have
\(
    \dist(x,\Pi\textbf{C})
    \le
    \dist(x,\textbf{C}).
\)
For \(x\notin K_R\), note that
\(
    g_{\Pi\textbf{C}}(x)=0.
\)
Therefore, for every \(x\in\mathbb R^d\),
\[
    \dist(x,\textbf{C})^z
    \ge
    g_{\Pi\textbf{C}}(x).
\]
Consequently,
\[
    \frac1n\sum_{i=1}^n \dist(X_i,\textbf{C})^z
    \ge
    \frac1n\sum_{i=1}^n g_{\Pi\textbf{C}}(X_i).
\]
Since \(\Pi\textbf{C}\in(K_R^+)^k\),  we deduce by equation~\eqref{eq:gC_positive_compact}, that on the event $\mathcal{E}$:
\[
    \frac1n\sum_{i=1}^n \dist(X_i,\textbf{C})^z
    \ge
    \frac{q}{2}.
\]
Since \(\textbf{C}\in(\mathbb R^d)^k\) is arbitrary, we deduce that on the event $\mathcal{E}$:
\[
    \inf_{1\le |\textbf{C}|\le k}
    \frac1n\sum_{i=1}^n \dist(X_i,\textbf{C})^z
    \ge
    \frac q2.
\]

Thus, for every \(n\),
\[
    \mathbb P\left(
        \inf_{1\le |\textbf{C}|\le k}
        \frac1n\sum_{i=1}^n \dist(X_i,\textbf{C})^z
        \ge
        \frac q2
    \right)
    \ge
    1-
    \left(1+\frac{4R}{\delta}\right)^{kd}
    \exp\left(
        -\frac{nq^2}{32(3R)^{2z}}
    \right).
\]

Define
\[
    n_0
    :=
    \left\lceil
        \frac{64(3R)^{2z}}{q^2}
        kd\log\left(1+\frac{4R}{\delta}\right)
    \right\rceil.
\]
If \(n\ge n_0\), then
\[
    kd\log\left(1+\frac{4R}{\delta}\right)
    \le
    \frac{nq^2}{64(3R)^{2z}}.
\]
Hence
\[
    \left(1+\frac{4R}{\delta}\right)^{kd}
    \exp\left(
        -\frac{nq^2}{32(3R)^{2z}}
    \right)
    \le
    \exp\left(
        -\frac{nq^2}{64(3R)^{2z}}
    \right).
\]
Therefore the theorem holds with
\[
    c := \frac{q^2}{64(3R)^{2z}}.
\]
Since \(R\) is an explicit function of \(z,q,\sigma\), the constant $n_0$ is an explicit functions of \(k,d,z,q,\sigma\). Similarly $c$ is an explicit function of $z,q,\sigma$.
\end{proof}

\section{On Polish spaces.}\label{app:polish}
A Polish space is a topological space whose topology can be induced by a
complete separable metric. This class is broad enough for many standard
spaces used in analysis, probability, statistics, and data science. In
particular, every compact metrizable space is Polish
\cite[Corollary~3.3.2]{garling2017analysis}. Therefore, whenever the data
domain is modeled as a compact metric space, the assumption that the
underlying space is Polish adds no extra restriction.

This includes many common finite-dimensional data domains. For example,
the cube $[0,1]^d$ is compact and metric, hence Polish. A fixed-resolution normalized image space can be identified with $[0,1]^3$
so it is also a compact metric space and hence Polish.

In fact, any finite label space with the discrete metric is a compact metric and hence Polish. 

Similarly, the
probability simplex
\[
    \Delta_{k-1}
    =
    \left\{
        p \in [0,1]^k : \sum_{i=1}^k p_i = 1
    \right\}
\]
is a closed subset of the compact cube $[0,1]^k$, so it is compact metric
and hence Polish. Thus, the Polish-space framework is general enough to
cover many compact metric spaces that arise naturally in data science
applications.

\section{The Average Squared Distance Objective}
\label{app:avg-squared-dist}

Consider an alternative but closely related objective to the standard $k$-means
clustering.
Instead of the minimum distance to the nearest center, define the
\emph{average squared distance} to a $k$-center set
$C=\{c_1,\dots,c_k\}\subset\mathbb{R}^d$ as
\begin{equation}
f_C(x) = \frac{1}{k}\sum_{c\in C}\|x-c\|^2.
\label{eq:avg-squared-dist}
\end{equation}

\begin{proposition}
\label{prop:avg-sq-dim}
Let $\mathcal{X}=\mathbb{R}^d$ and let
$\mathcal{F}=\bigl\{f_C:\,C\subset\mathbb{R}^d,\,|C|=k\bigr\}$
be the class of functions defined in \eqref{eq:avg-squared-dist}.
The linear span of $\mathcal{F}$ has finite dimension $D=d+2$, independently
of $k$, thereby satisfying Assumption~\textnormal{(A.1)} of
\cite{bardenet2024small}.
\end{proposition}

\begin{proof}
Fix any center set $C=\{c_1,\dots,c_k\}$.
Expanding and averaging the squared distances gives
\[
f_C(x)
= \|x\|^2 - 2\bar{c}^\top x + R^2,
\]
where $\bar{c}=\frac{1}{k}\sum_{i=1}^k c_i$ is the centroid and
$R^2=\frac{1}{k}\sum_{i=1}^k\|c_i\|^2$ is the mean squared norm of the
centers.
Hence every $f_C\in\mathcal{F}$ is an affine combination of the $d+2$
linearly independent functions
$h_0(x)=\|x\|^2$,\; $h_j(x)=x_j$ for $j=1,\dots,d$,\; and $h_{d+1}(x)=1$.
Therefore $\dim(\operatorname{span}_{\mathbb{R}}(\mathcal{F}))=d+2<\infty$.
\end{proof}

Because $\mathcal{F}$ lies in a fixed $(d+2)$-dimensional function space, the
first part of Theorem~6 from \cite{bardenet2024small} applies.
With high probability over the dataset $P\sim\mu^n$, the DPP-sampled weighted
coreset $\mathcal{S}$ of expected size $m$ satisfies
\[
\mathbb{P}_{\mathcal{S}}\!\left(
  \exists\,f\in\mathcal{F}:\,
  \left|\frac{L_{\mathcal{S}}(f)}{L(f)}-1\right|\ge\varepsilon
\right)
\le 2\exp\!\bigl(6(d+2)-C'\varepsilon^2 m^{1+1/d}\bigr)
\]
for $\varepsilon=O(m^{-1/d})$, where $C'>0$ is a universal constant.
Remark~6.1 of \cite{bardenet2024small} then gives, with probability tending
to $1$,
\[
\forall\,f\in\mathcal{F},\qquad
\left|\frac{L_{\mathcal{S}}(f)}{L(f)}-1\right|
\le m^{-(1/2+1/(2d))+o(1)}.
\]
Setting the right-hand side equal to $\varepsilon$ gives the coreset size
\[
m = O\!\left(\varepsilon^{-\frac{2d}{d+1}}\right),
\]
hiding constants and lower-order logarithmic factors.
Crucially, this bound is \emph{independent of $k$}.

\section{Supplementary Benchmark Plots}
\label{app:supp-benchmark}

This appendix collects the full benchmark diagnostics for all nine datasets evaluated in Section~\ref{sec:experiments}: the six synthetic datasets (Uniform 2D, Trimodal 2D, Gaussian Blobs 2D, Gaussian Blobs 3D, Anisotropic Gaussians 2D, BIRCH Benchmark) and the three real-world datasets (3D Road Network, UCI Adult PCA-2, Twitter GPS). For each dataset, we show three complementary metrics alongside the average relative error reported in the main text. Together, these figures allow the reader to distinguish variance effects from slope effects, and to compare average-case against worst-case behavior across methods and coreset sizes.

\paragraph{Metric definitions.}
Throughout, we fix $N_{\mathrm{REP}}=150$ independent coreset draws $\mathcal{S}^{(1)},\dots,\mathcal{S}^{(N_{\mathrm{REP}})}$ and a fixed family of $N_{\mathrm{QUERY}}=150$ random query center sets $C_1,\dots,C_{N_{\mathrm{QUERY}}}$ drawn once per dataset-$k$ pair.
\begin{itemize}
  \item \textbf{Supremum variance} (``Sup.\ variance''). For each query $C_j$, let
    $\widehat\sigma^2_j(m) := \frac{1}{N_{\mathrm{REP}}-1}\sum_{r=1}^{N_{\mathrm{REP}}} \bigl(L_{\mathcal{S}^{(r)}}(C_j) - \overline{L}(C_j)\bigr)^2$
    be the empirical variance of the coreset estimator over repetitions, where $\overline{L}(C_j)$ is its sample mean. The supremum variance is then
    \[
      \sup_{j\in[N_{\mathrm{QUERY}}]} \widehat\sigma^2_j(m).
    \]
    This measures the worst-case spread of the estimator across the query family, and is the empirical analog of $V = \sup_\theta \mathrm{Var}[n^{-1}L_{\mathcal{S}}(f_\theta)]$ from Theorem~\ref{thm:kz-clustering-dpp}.
  \item \textbf{Average relative error} (``Avg.\ rel.\ error''). This is $\overline{\mathrm{RelErr}}(m)$ as defined in Section~\ref{sec:experiments}, reproduced here for ease of comparison.
  \item \textbf{$0.9$-quantile supremum relative error} ($Q_{\mathcal{S}}(0.9)$). For each repetition $r$, let
    $E^{(r)}(m) := \sup_{j\in[N_{\mathrm{QUERY}}]} \frac{|L_{\mathcal{S}^{(r)}}(C_j) - L(C_j)|}{L(C_j)}$
    be the worst-case (supremum) relative error over all queries. The metric $Q_{\mathcal{S}}(0.9)$ is the $0.9$-quantile of $E^{(1)}(m),\dots,E^{(N_{\mathrm{REP}})}(m)$:
    \[
      Q_{\mathcal{S}}(0.9) := \mathrm{quantile}_{0.9}\!\bigl(E^{(1)}(m),\dots,E^{(N_{\mathrm{REP}})}(m)\bigr).
    \]
    This robustly captures worst-query performance: it reports the supremum error that is exceeded in only $10\%$ of independent coreset draws, and is thus more stable than the hard maximum over repetitions.
\end{itemize}

\newcommand{\suppbenchfig}[5]{%
\begin{figure}[p]
  \centering
  \begin{subfigure}[b]{0.32\textwidth}
    \includegraphics[width=\textwidth]{plots/#3}
    \caption*{Sup.\ variance}
  \end{subfigure}\hfill
  \begin{subfigure}[b]{0.32\textwidth}
    \includegraphics[width=\textwidth]{plots/#4}
    \caption*{Avg.\ rel.\ error}
  \end{subfigure}\hfill
  \begin{subfigure}[b]{0.32\textwidth}
    \includegraphics[width=\textwidth]{plots/#5}
    \caption*{$Q_{\mathcal S}(0.9)$}
  \end{subfigure}
  \caption{Supplementary metrics for \textbf{#1} ($d=#2$, $k=5$).}
\end{figure}}

\suppbenchfig{Uniform 2D}{2}{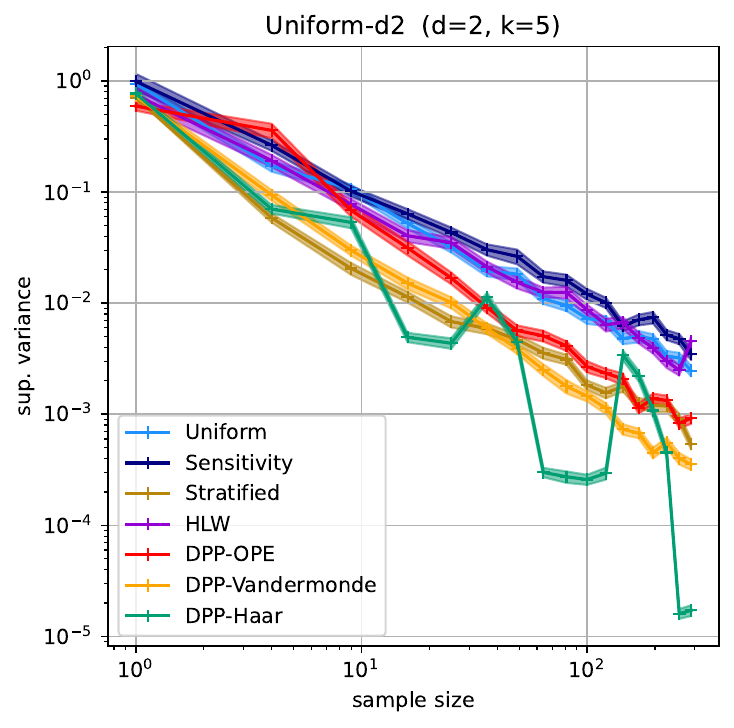}{Uniform-d2_k5_avg.pdf}{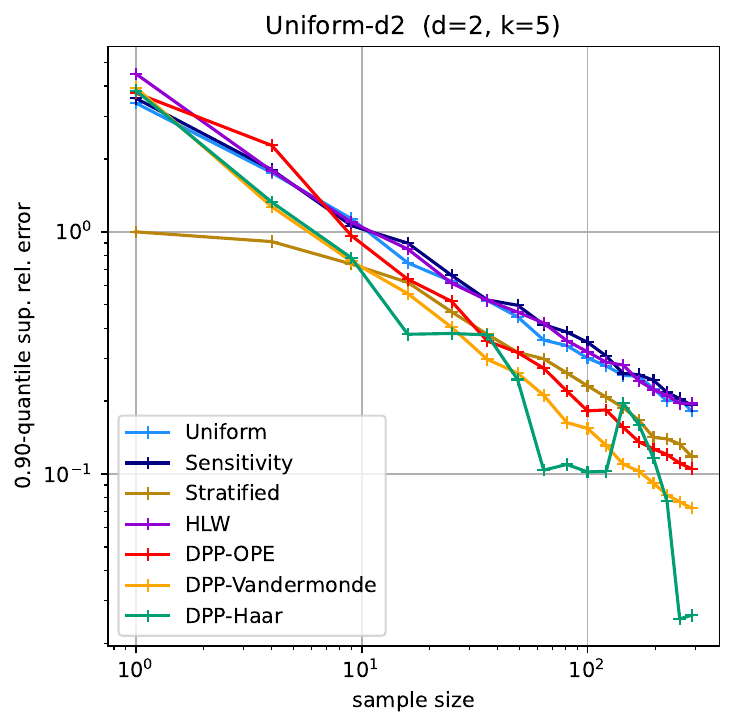}
\suppbenchfig{Trimodal 2D}{2}{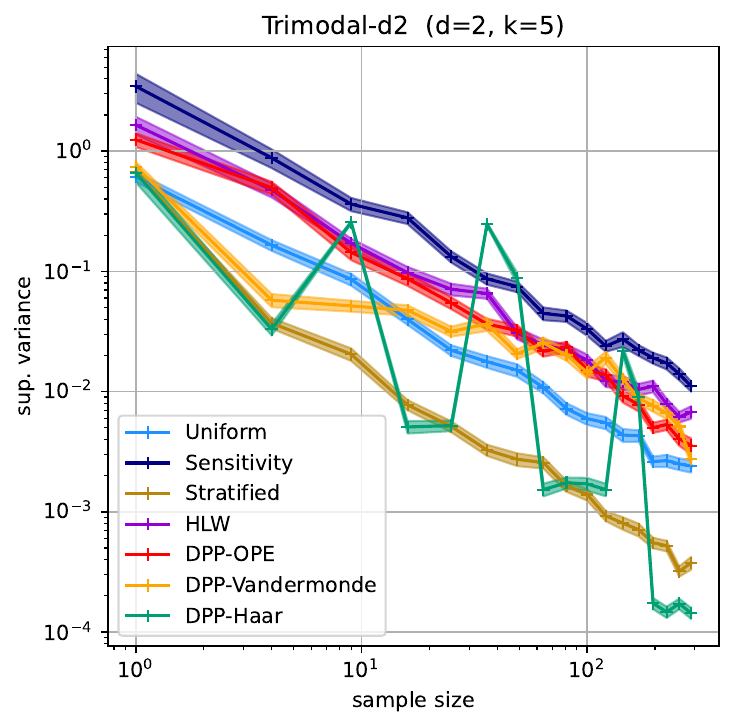}{Trimodal-d2_k5_avg.pdf}{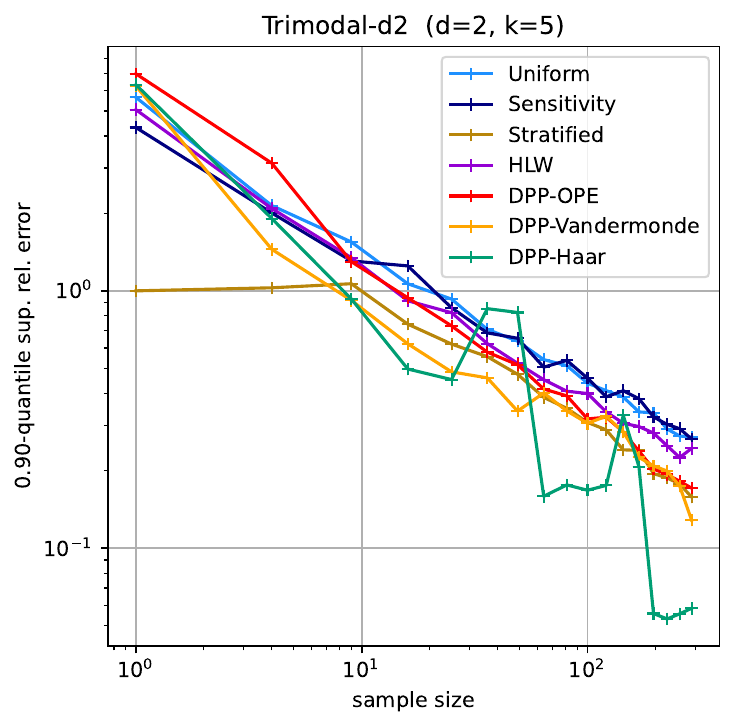}
\suppbenchfig{Gaussian Blobs 2D}{2}{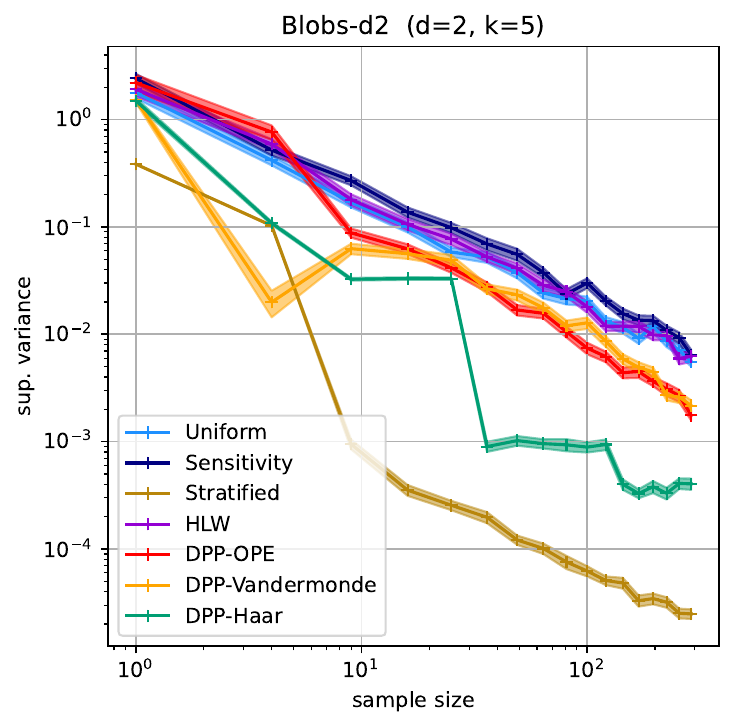}{Blobs-d2_k5_avg.pdf}{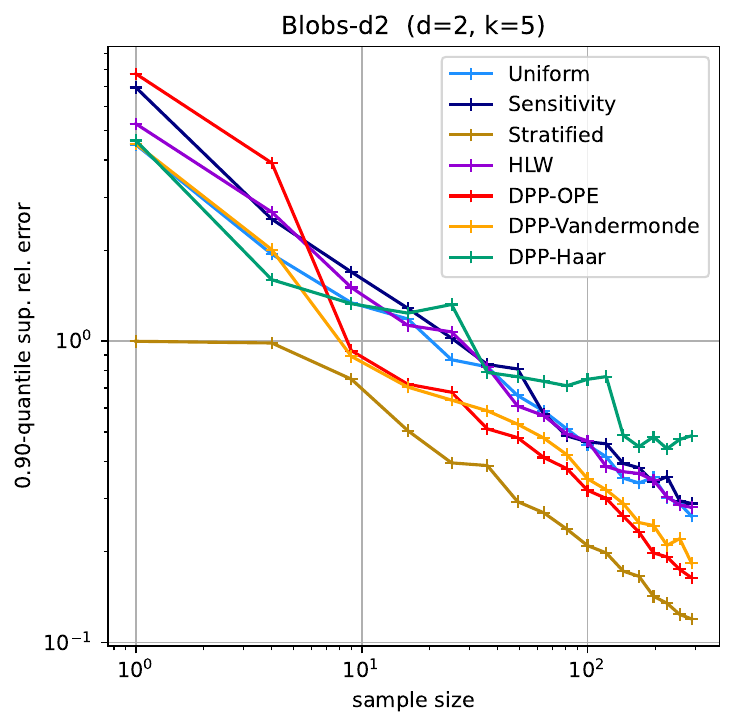}
\suppbenchfig{Gaussian Blobs 3D}{3}{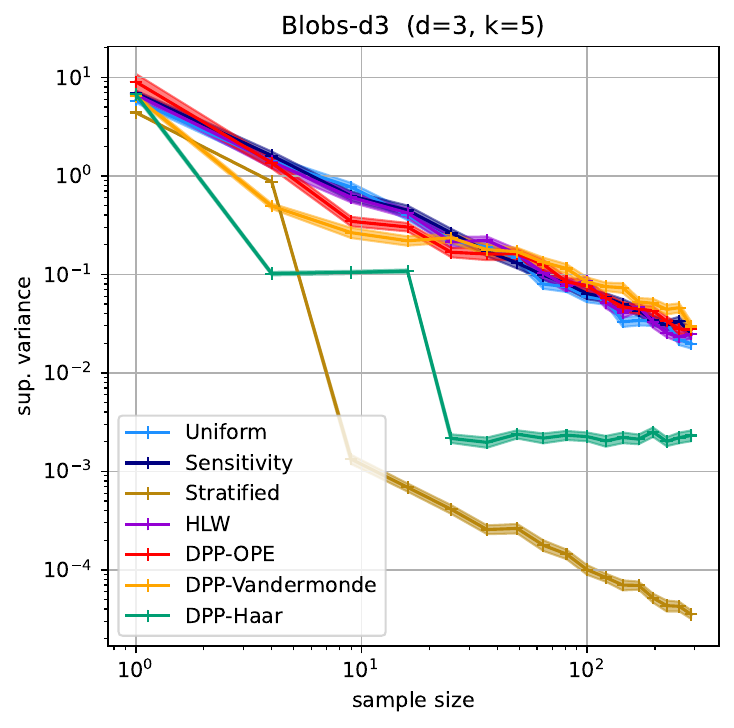}{Blobs-d3_k5_avg.pdf}{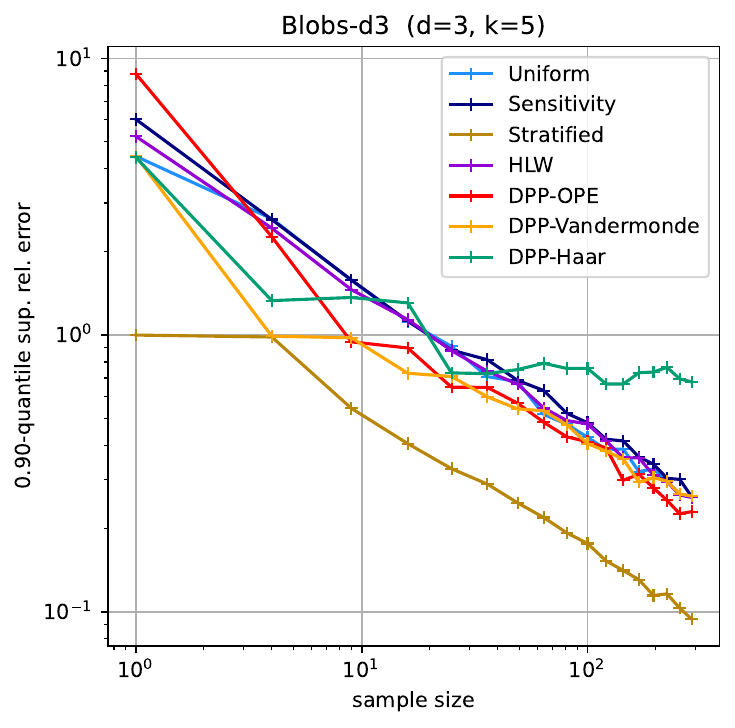}
\suppbenchfig{Anisotropic Gaussians 2D}{2}{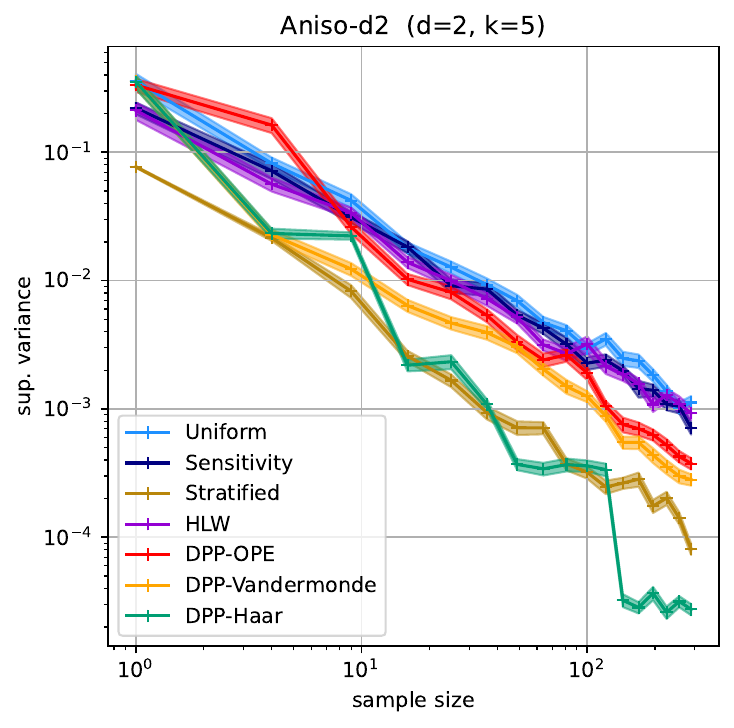}{Aniso-d2_k5_avg.pdf}{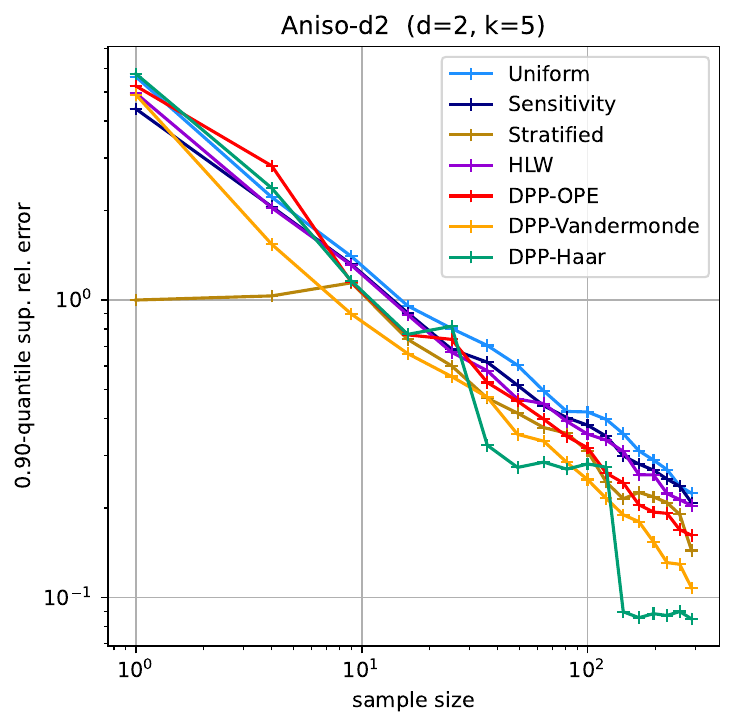}
\suppbenchfig{BIRCH Benchmark}{2}{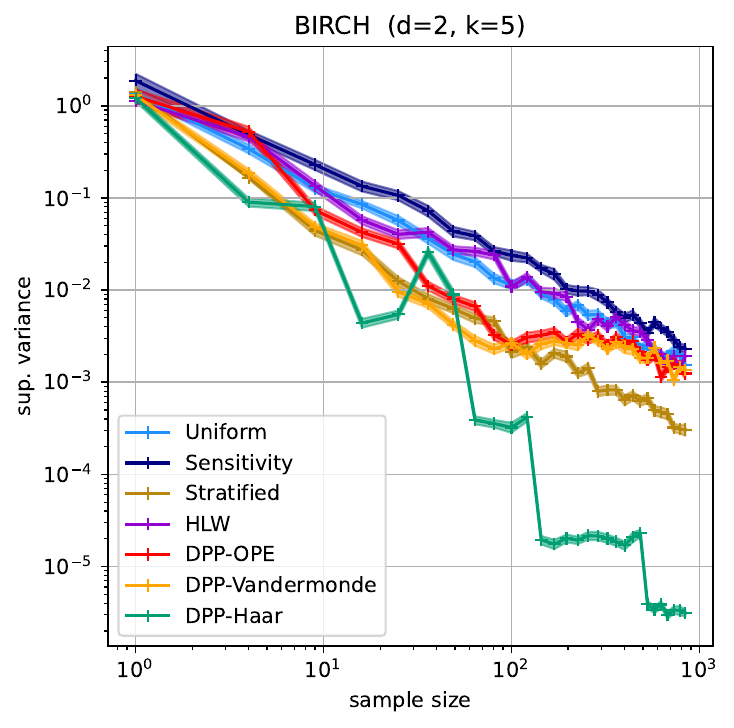}{BIRCH_k5_avg.pdf}{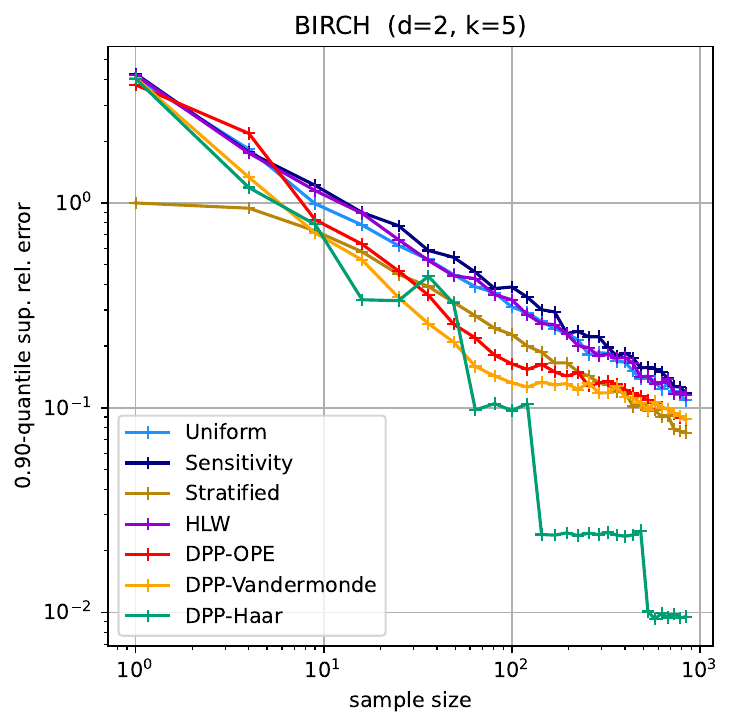}

\suppbenchfig{3D Road Network}{3}{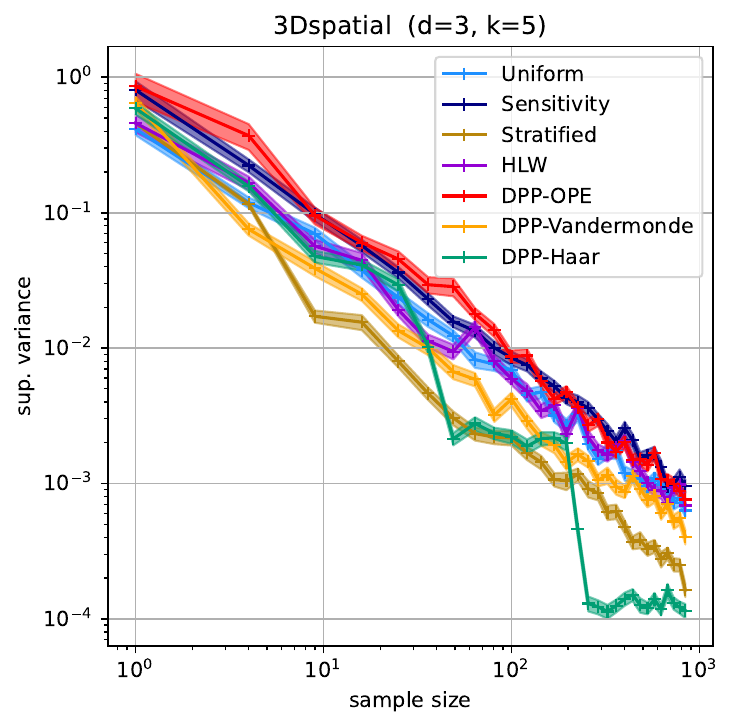}{3Dspatial_k5_avg.pdf}{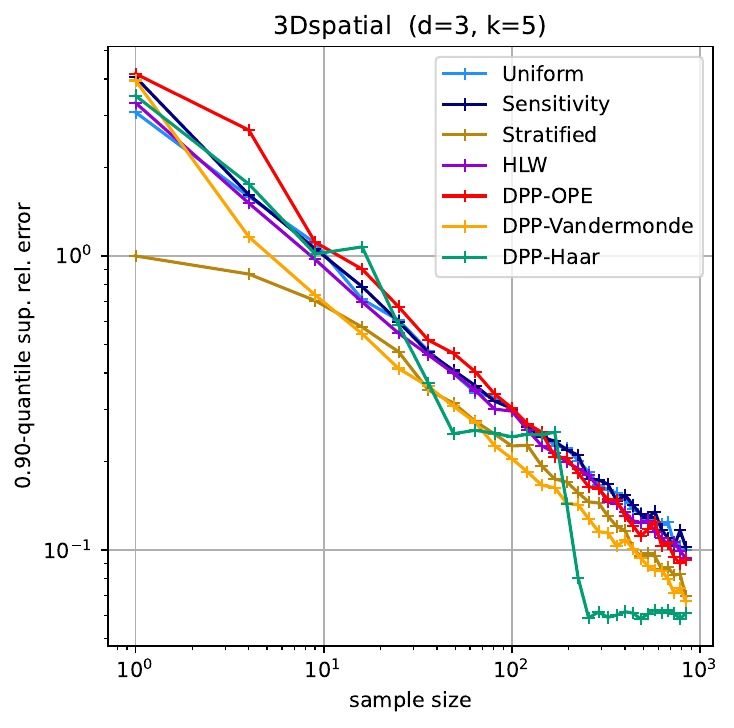}
\suppbenchfig{UCI Adult (PCA-2)}{2}{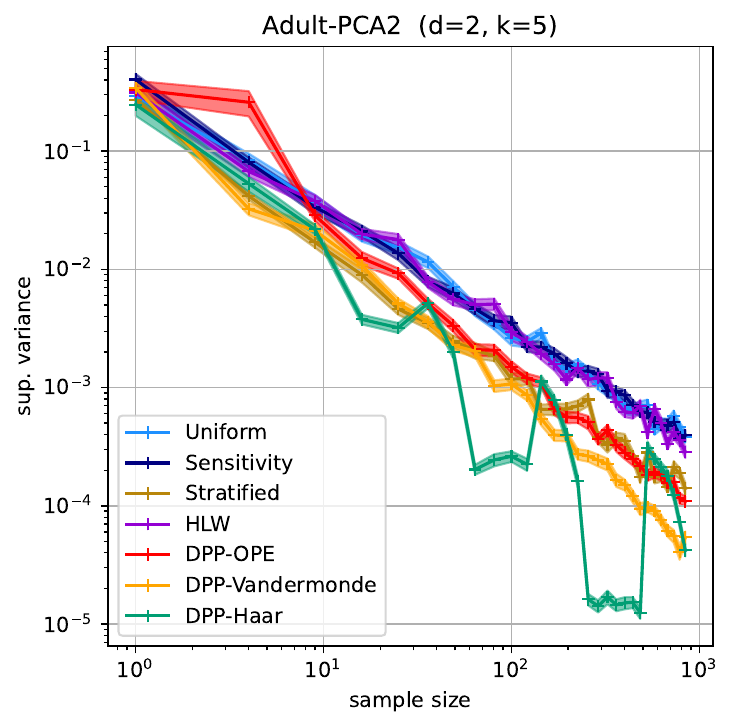}{Adult-PCA2_k5_avg.pdf}{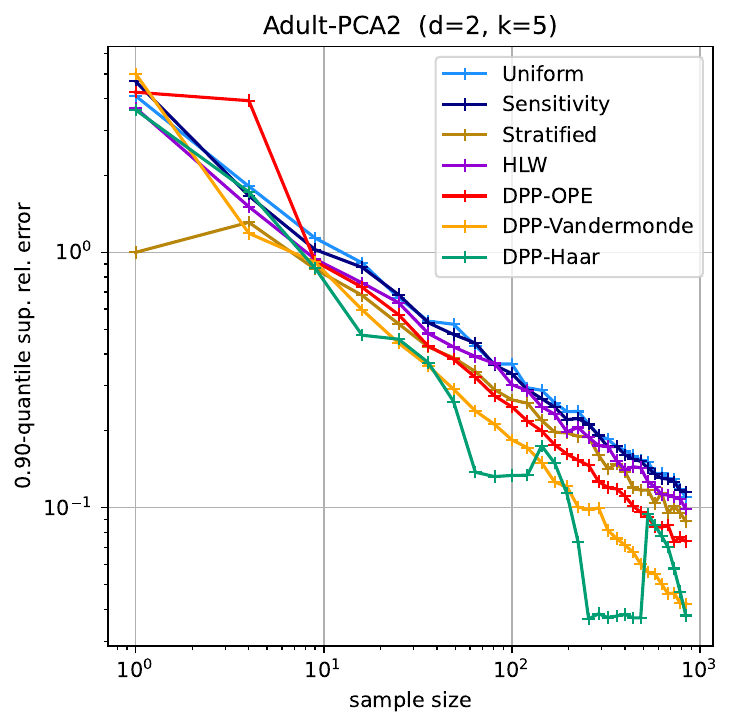}
\suppbenchfig{Twitter GPS}{2}{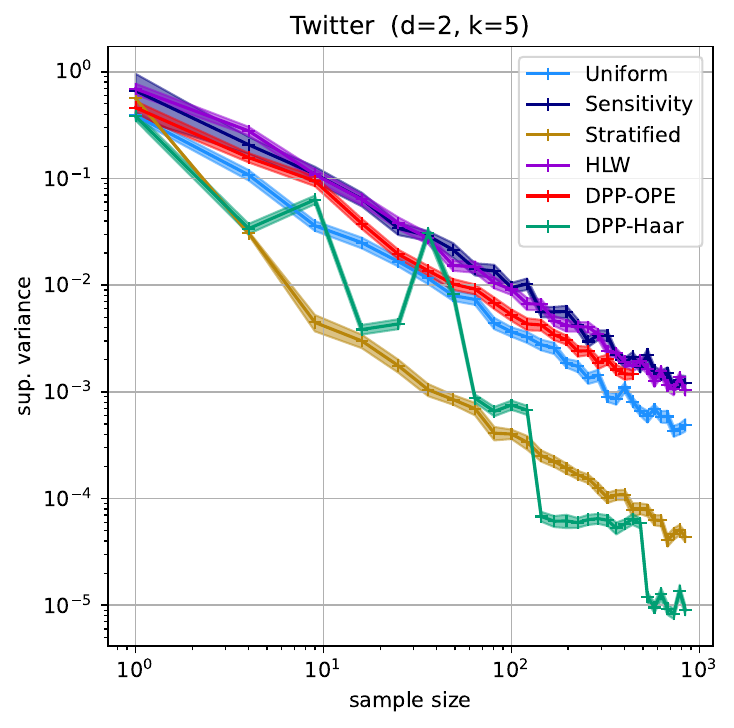}{Twitter_k5_avg.pdf}{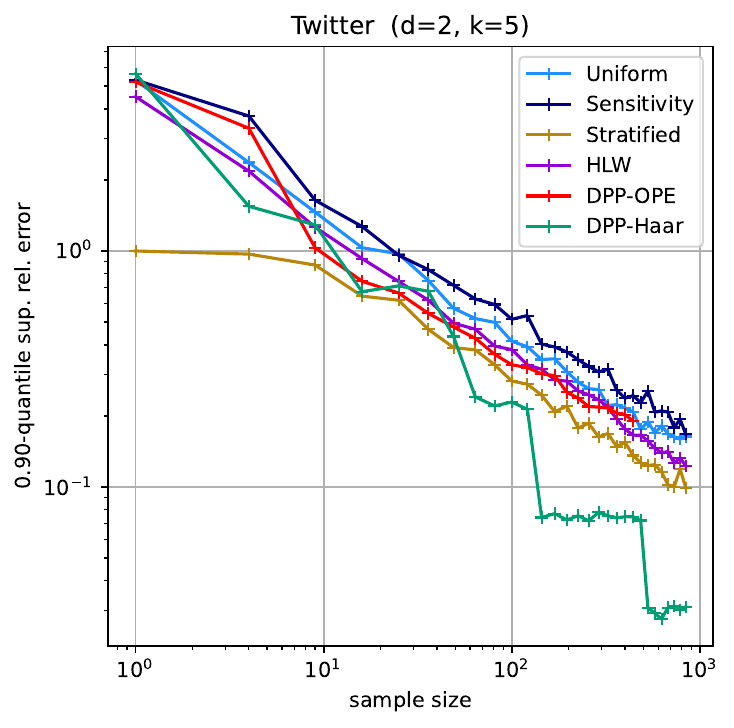}
\newpage
\section{License}\label{app:license}

\begin{itemize}
        \item \textbf{Twitter GPS dataset}~\cite{chan2018twitter}. Collected and publicly released by Chan, Guerqin, and Sozio alongside their WWW~2018 paper ``Fully Dynamic $k$-Center Clustering'' (ACM DOI: \href{https://doi.org/10.1145/3178876.3186124}{10.1145/3178876.3186124}; HAL preprint: \href{https://hal.science/hal-02315471}{hal-02315471}), available at \url{https://github.com/fe6Bc5R4JvLkFkSeExHM/k-center}. The dataset contains only GPS coordinates (longitude, latitude) and timestamps; no tweet text is included, in accordance with Twitter's privacy policy. The paper and accompanying dataset are published under the \textbf{Creative Commons Attribution 4.0 International (CC BY 4.0)} license, as stated on the paper's title page.

        \item \textbf{3D Road Network dataset}~\cite{kaul2013roadnetwork}. Obtained from the UCI Machine Learning Repository under \textbf{CC BY 4.0}.

        \item \textbf{UCI Adult dataset}~\cite{becker1996adult}. Obtained from the UCI Machine Learning Repository (DOI: \href{https://doi.org/10.24432/C5XW20}{10.24432/C5XW20}) under \textbf{CC BY 4.0}.

        \item \textbf{BIRCH Benchmark}~\cite{zhang1997birch}. Synthetic benchmark generated by our own code following the procedure described in~\cite{zhang1997birch}; no third-party data are used.

        \item \textbf{DPPy library}~\cite{gautier2019dppy}. Used for \texttt{MultivariateJacobiOPE}; released under the \textbf{MIT License} (\url{https://github.com/guilgautier/DPPy}).

        \item \textbf{scikit-learn}~\cite{pedregosa2011scikit}. Used for PCA and baseline clustering utilities; released under the \textbf{BSD 3-Clause License} (\url{https://github.com/scikit-learn/scikit-learn}).

        \item All remaining synthetic datasets (Uniform, Trimodal, Gaussian Blobs, Anisotropic Gaussians) are generated entirely by our own code and involve no third-party data or assets.
    \end{itemize}

\end{document}